\documentclass{article}

\usepackage{microtype}
\usepackage{graphicx}
\usepackage{tabularx}
\usepackage{subfigure}
\usepackage{subcaption}
\usepackage{booktabs} 
\usepackage[T1]{fontenc}
\usepackage{hyperref}

\usepackage[accepted]{icml2026}

\usepackage{amsmath}
\usepackage{amssymb}
\usepackage{mathtools}
\usepackage{amsthm}
\usepackage{multirow}
\usepackage{booktabs}

\usepackage[capitalize,noabbrev]{cleveref}

\theoremstyle{plain}
\newtheorem{theorem}{Theorem}[section]

\newtheorem{corollary}[theorem]{Corollary}
\theoremstyle{definition}

\theoremstyle{remark}

\usepackage[textsize=tiny]{todonotes}

\icmltitlerunning{Enhancing LLMs for Graph Tasks via Graph-aware LoRA Generation}

\begin{document}

\twocolumn[
  \icmltitle{Enhancing LLMs for Graph Tasks via Graph-aware LoRA Generation}



  \icmlsetsymbol{equal}{*}

  \begin{icmlauthorlist}
    \icmlauthor{Junshu Sun}{ict,ucas}
    \icmlauthor{Wanxing Chang}{ali}
    \icmlauthor{Qingming Huang}{ucas}
    \icmlauthor{Shuhui Wang}{ict}
  \end{icmlauthorlist}

  \icmlaffiliation{ict}{State Key Lab. of AI Safety, Institute of Computing Technology, Chinese Academy of Sciences, Beijing, China}
  \icmlaffiliation{ucas}{University of Chinese Academy of Sciences, Beijing, China}
  \icmlaffiliation{ali}{DAMO Academy, Alibaba Group, Hangzhou, China}

  \icmlcorrespondingauthor{Shuhui Wang}{wangshuhui@ict.ac.cn}

  \icmlkeywords{Machine Learning, ICML}

  \vskip 0.3in
]



\printAffiliationsAndNotice{}  

\begin{abstract}
Graph neural networks (GNNs) tightly couple their input-output parameters to dataset-specific feature spaces and target sets, exhibiting limited transferability across different datasets. In contrast, language models (LMs) generalize flexibly via a unified input-output interface, motivating recent attempts to adapt LMs to graph tasks. However, existing methods struggle to encode whole-graph information, leading to potential information loss and suboptimal graph understanding. In this work, we propose a novel weight-level information injection paradigm for adapting LMs to graph tasks. This paradigm injects whole-graph information by generating task-specific weight updates that interact directly with hidden representations. Instantiating this paradigm following low-rank adaptation (LoRA), we introduce GaRA, a Graph-aware LoRA generation model. GaRA constructs low-rank weight updates conditioned on the original graph structures and constrains the norm of the generated updates, thus injecting whole-graph information and avoiding the optimization bias in the weight generation. Empirical studies demonstrate that GaRA consistently outperforms baselines on zero-shot graph learning tasks.
\end{abstract}

\begin{figure}[htb]
\centering
\includegraphics[width=0.85\linewidth]{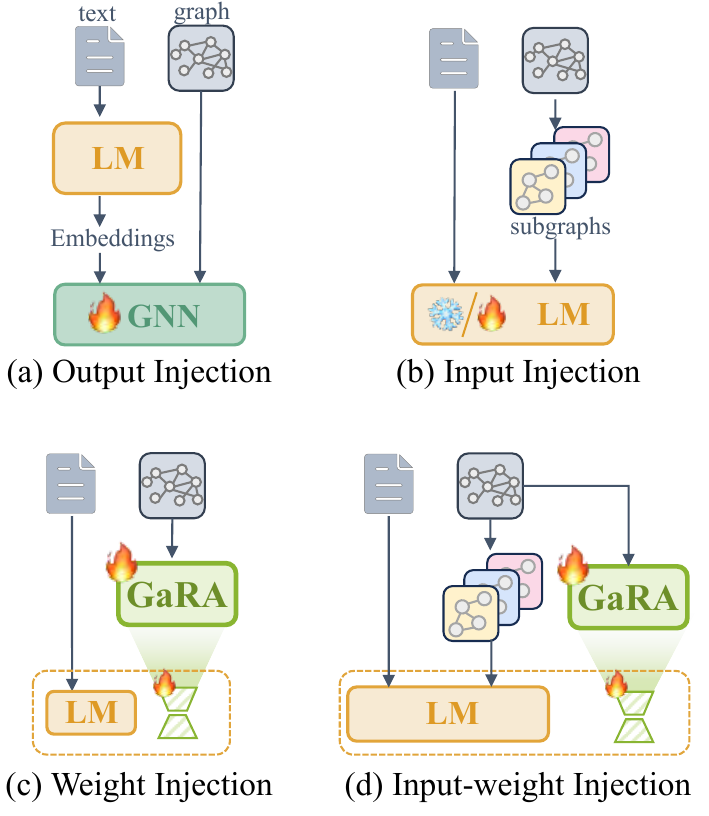}
\caption{\textbf{Comparison between Graph Injection Paradigms.} (a) Output-level injection paradigm employs GNNs as the model predictor, failing to tackle graphs with diverse labels. (b) Input-level injection paradigm transforms input graphs into subgraph sequences, overlooking higher-level information in the original graphs and leading to potential information loss. (c) Our weight-level injection paradigm enables LMs to learn from whole-graph information while maintaining the ability to tackle diverse tasks. (d) In practice, weight-level injection is complementary to the input-level injection, alleviating potential information loss.}
\label{fig:teaser}
\end{figure}

\section{Introduction}
Graph neural networks (GNNs) have emerged as effective models for supervised graph representation learning, addressing tasks such as node classification~\cite{kipf_SemiSupervisedClassificationGraph_2017,sun_RelievingAggregatingEffect_2025}, link prediction~\cite{zhang_LinkPredictionBased_2018}, and graph classification~\cite{ying_TransformersReallyPerform_2021,sun_AllinARow_2023} across diverse application scenarios~\cite{li_DeepGCNsMakingGCNs_2021, zhang_DynamicGraphMessage_2022, bessadok_GraphNeuralNetworks_2023}. However, the parameterization of GNNs is inherently tied to the number of input feature dimensions and output targets ({\it e.g.}, labels or attributes) in the training data~\cite{zhao_FullyinductiveNodeClassification_2024}. As graphs from different datasets exhibit heterogeneous feature spaces or target sets~\cite{chen_TextspaceGraphFoundation_2024}, pre-trained GNNs cannot be directly transferred to downstream graphs without modifying model parameters or architectures~\cite{zhao_FullyinductiveNodeClassification_2024}. In contrast, language models (LMs) operate on a unified textual interface, taking textual sequences as both inputs and outputs~\cite{chen_LLaGALargeLanguage_2024,li_ZeroGInvestigatingCrossdataset_2024}. By varying the length and content of input sequences, LMs naturally accommodate heterogeneous feature descriptions and produce predictions via textual generalization, enabling flexible transfer across tasks and domains. 

Motivated by the flexibility of LMs, recent studies have explored their application to graph learning~\cite{chen_LLaGALargeLanguage_2024}. A major challenge in this line of work is enabling the model to effectively capture graph topology~\cite{chen_LLaGALargeLanguage_2024}. Mainstream solutions can be broadly categorized into {\it output-level} and {\it input-level} injecting methods. Output-level injection leverages GNNs to refine LM representations and produce final predictions~\cite{liu_OneAllTraining_2023}. However, similar to traditional GNNs, these methods remain tightly coupled to task-specific target sets and thus exhibit limited transferability across different downstream tasks. Differently, input-level injection incorporates structural information directly into LM inputs by transforming graph structures into prompt sequences, and subsequently fine-tunes pre-trained LMs to adapt to these prompts~\cite{chen_LLaGALargeLanguage_2024,wang_UniGTEUnifiedGraph_2025}. Nevertheless, due to the fully connected self-attention mechanism, LMs incur $O(n^2)$ computational complexity for an input sequence of length $n$. As a result, graph-based prompts are typically restricted to small subgraphs, discarding the remaining structural context. This limitation overlooks higher-level information encoded in the original graph and potentially leads to information loss and biased graph understanding.

To alleviate information loss, direct solutions include expanding the subgraph size or concatenating the global graph feature to the input sequence, yet increasing the computational overhead quadratically. This suggests that simply augmenting information volume at the input level is not a viable solution. In this paper, we propose a novel {\it weight-level} injection paradigm for adapting LMs to graph tasks. Specifically, the injection is realized by generating the weight matrix with the whole-graph information and applying the resulting weights to the representations in the hidden layers. Our weight-level injection complements the subgraph information injected at the input level, alleviating potential information loss while avoiding a quadratic increase in computational cost.

We then implement this weight-level paradigm with low-rank adaptation (LoRA), proposing the \textbf{G}raph-\textbf{a}ware L\textbf{oRA} generation method, named GaRA, including a LoRA-style weight generation module and a norm-rescaling module. The weight generation module generates a low-rank matrix to construct weight updates, injecting whole-graph information into the LM backbone. The norm-rescaling module is proposed to tackle the direction-preserving norm amplification problem, which we prove as an architectural optimization bias during fine-tuning. In the empirical analysis, we show that GaRA can consistently surpass baseline methods on zero-shot transfer tasks. Our contributions can be summarized as follows:
\begin{itemize}
    \item We design a novel weight-level information-injecting paradigm, which preserves the whole-graph information for LMs on graph tasks.
    \item We propose a graph-aware weight generation method, GaRA, under the weight-level paradigm, which alleviates potential information loss and tackles the architectural optimization bias problem.
    \item We demonstrate the advantages of GaRA on zero-shot graph learning tasks, which surpass baselines on different transfer tasks.
\end{itemize}

\section{Related Work}
LMs have demonstrated predominant generalizability across a wide range of tasks~\cite{openai_GPT4oSystemCard_2024}. In light of this, substantial efforts have been devoted to adapting LMs to graph tasks~\cite{chen_LLaGALargeLanguage_2024}. However, Wang et al., \yrcite{wang_CanLanguageModels_2023} suggest that pure LMs struggle to capture structural relations in graphs. To address this limitation, recent studies incorporate graph structures into LMs~\cite{chen_LLaGALargeLanguage_2024}, enabling more effective adaptation to graph tasks. Existing work can be broadly categorized into output-level and input-level injecting methods.

\subsection{Output-level Injection}
In these methods, LMs typically serve as text encoders, while GNNs further enhance the representations and make final predictions~\cite{liu_OneAllTraining_2023,chen_TextspaceGraphFoundation_2024}. Owing to the strong textual understanding capabilities, LMs facilitate semantic alignment across diverse real-world graphs, thereby reducing domain gaps in applications such as biochemical graphs~\cite{morris_TUDatasetCollectionBenchmark_2020}, social networks~\cite{lim_LargeScaleLearning_2021}, and e-commerce graphs~\cite{hu_OpenGraphBenchmark_2020}. Given the LM-encoded textural features, GNNs further inject structural features through message passing, ensuring graph-awareness of the final predictions. Nevertheless, the GNN-based predictors fail to tackle graphs with diverse labels~\cite{zhao_FullyinductiveNodeClassification_2024}, leading to limited adaptability of the output-level injection. Due to this limitation, we follow the input-level injecting methods to construct our framework in this paper.

\subsection{Input-level Injection}
In contrast to output-level methods, input-level injection adopts LMs as predictors. Structural information is directly injected into the LM inputs by concatenating graph prompts, thereby enabling graph-aware learning. Typically, the graph prompts can be derived from GNN outputs~\cite{wang_LLMsZeroshotGraph_2024a,zhang_GraphTranslatorAligningGraph_2024a,lv_GraphPrompterMultiStageAdaptive_2025,he_UniGraphLearningUnified_2025} or graph compression strategies~\cite{chen_LLaGALargeLanguage_2024,wang_UniGTEUnifiedGraph_2025}. 

\paragraph{GNN-based Prompting.}
GNN-based prompting strategy encodes graph structures into prompt embeddings with various GNN models~\cite{kipf_SemiSupervisedClassificationGraph_2017,hamilton_InductiveRepresentationLearning_2017,sun_DynamicMessagePassing_2024}. However, to ensure computational efficiency and training stability, these methods commonly freeze the GNN during LM tuning~\cite{wang_LLMsZeroshotGraph_2024a,li_ZeroGInvestigatingCrossdataset_2024}. In consequence, the resulting graph-dependent signals cannot be optimized for downstream objectives, which may lead to suboptimal representations and exacerbate distribution shifts between pre-training and downstream data~\cite{wang_UniGTEUnifiedGraph_2025}. To address this problem, GOFA~\cite{kong_GOFAGenerativeOneAll_2024} interleaves GNN layers into a frozen pre-trained LLM, adjusting the prompt signal accordingly but also incurring intensive computational requirement.

\paragraph{Graph Compression Prompting.}
Unlike GNN-based prompting, graph compression strategies leverage the relational modeling capacity of pre-trained LMs, such as large language models (LLMs), to learn prompts from node sequences. A key challenge in this paradigm is transforming input graphs into sequences while preserving structural information. To this end, LLaGA~\cite{chen_LLaGALargeLanguage_2024} and ZeroG~\cite{li_ZeroGInvestigatingCrossdataset_2024} generate node sequences via neighborhood sampling for each node \footnote{ZeroG also applies message passing after the LM. Since this process introduces no learnable parameters, we categorize ZeroG as an input-level injecting method.}. However, the sampling strategy is permutation-variant, making the compressed representations sensitive to the ordering of the sampled nodes. In contrast, UniGTE~\yrcite{wang_UniGTEUnifiedGraph_2025} restores permutation invariance by fixing relative positional encodings between nodes to zero and injecting graph structure via attention biases.

Beyond the characteristics discussed above, due to constraints on model complexity and the goal of task unification, both GNN-based methods and graph compression strategies reduce all tasks to the graph level by extracting an $n$-hop subgraph centered on the target node or edge. This process overlooks higher-level information encoded in the original graph, potentially leading to information loss in downstream tasks. Different from these methods that solely operate at the input level, we propose a weight-level paradigm, injecting whole-graph information via weight generation.

\subsection{Weight-level Injection}
Weight-level injection has emerged as a promising paradigm for adapting LMs to the visual data. For example, VLoRA~\cite{ma_VisualPerceptionLarge_2024} extracts informative features via learnable queries and transforms them into low-rank weight updates. Adapting this paradigm to graph tasks involves two challenges. First, graph datasets typically contain much less training data than vision datasets, making weight generation more prone to optimization bias during fine-tuning. Second, graph data encodes rich relational structures and multi-scale information, making it challenging to inject sufficient features via a small number of generated weight updates. To address these challenges, GaRA employs norm rescaling in weight injection to avoid optimization bias. The global features captured by weight injection are combined with the local features from input injection, thereby enabling more sufficient feature injection.

\section{Preliminaries}
\subsection{Notations}
Given an input graph $\mathcal{G}=(\mathcal{V}, \mathcal{E})$, where $\mathcal{V}=\{v_1, \dots, v_n\}$ denotes the set of nodes with $|\mathcal{V}|=n$, and $\mathcal{E}=\{e_{i,j}\mid v_j \in \mathcal{N}(v_i)\}$ denotes the set of edges with $|\mathcal{E}|=m$. For a node $v\in \mathcal{V}$, let $\mathcal{N}(v)$ denote its one-hop neighborhood. Each node $v$ is associated with a feature vector $\mathbf{x}_v \in \mathbb{R}^{d_\mathtt{in}}$ and a textual description $t_v$, where $d_\mathtt{in}$ denotes the feature dimension. Stacking all node features yields the node feature matrix $\mathbf{X} = (\mathbf{x}_{v_1}, \dots, \mathbf{x}_{v_n})^\top \in \mathbb{R}^{n\times d_\mathtt{in}}$. Traditional GNNs are sensitive to the specific feature dimension $d_\mathtt{in}$ of different graphs. To tackle the heterogeneous feature spaces, pre-trained LMs~\cite{devlin_BERTPretrainingDeep_2019,reimers-2019-sentence-bert} are adopted to embed textual descriptions into $\mathbf{X}$, unifying input features in the embedding space. Graph structures can be represented by the adjacency matrix $\mathbf{A} \in \mathbb{R}^{n\times n}$, where $\mathbf{A}_{i,j} = 1$ if $v_j\in \mathcal{N}(v_i)$ and $0$ otherwise. 

\subsection{Model Fine-tuning}
During the adaptation of the LMs, fine-tuning strategies are adopted for pre-trained models. Without loss of generality, given a weight matrix $\hat{\mathbf{W}}\in\mathbb{R}^{p\times q}$ for a linear transformation $\mathbb{R}^{p}\mapsto\mathbb{R}^{q}$, the general adaptation process can be formulated as an additive modification to the original weights
\begin{equation}\label{eq:general}
    \texttt{(General)}\quad \mathbf{W} = \hat{\mathbf{W}} + \Delta\mathbf{W},
\end{equation}
where $\Delta\mathbf{W}$ denotes the weight update. Among different strategies, low-rank adaptation (LoRA) achieves pronounced efficiency and tuning effectiveness~\cite{hu_LoRALowRankAdaptation_2021}. Given a weight matrix $\hat{\mathbf{W}}\in\mathbb{R}^{p\times q}$, LoRA introduces two learnable low-rank matrices $\mathbf{L}\in\mathbb{R}^{p\times r}, \mathbf{R}\in\mathbb{R}^{q\times r}$ to enable parameter-efficient fine-tuning. The weight update can be reparameterized as
\begin{equation}\label{eq:lora}
    \texttt{(LoRA)}\quad \mathbf{W} = \hat{\mathbf{W}} + \Delta \mathbf{W} = \hat{\mathbf{W}} + \mathbf{LR}^\top
\end{equation}
When $r\ll \min(p,q)$, LoRA significantly reduces memory consumption and improves training efficiency by optimizing only the low-rank matrices while keeping $\hat{\mathbf{W}}$ fixed.

\section{Weight-level Graph Injection}
Despite improved efficiency and enhanced adaptability, input-level injection strategies discard higher-level information in the whole-graph context, potentially resulting in information loss and suboptimal adaptation of LMs to graph tasks. To alleviate information loss, we propose a novel weight-level paradigm, injecting whole-graph information via weight generation and complementary to the input-level subgraph information. 

\subsection{Task Description}
To achieve task unification, both node and edge tasks are reformulated by extracting task-specific subgraphs as graph-level inputs~\cite{tang_GraphGPTGraphInstruction_2024a}. Each task is then represented by a tuple $(\mathcal{G},\mathcal{G}_s,t_\texttt{task},t_\texttt{label})$, where $\mathcal{G}_s\subseteq\mathcal{G}$ is the subgraph associated with the task, $t_\texttt{task}$ is a textual task description specifying the prediction objective, and $t_\texttt{label}$ is the corresponding textual label. Graph-level tasks are treated as a special case where $\mathcal{G}_s=\mathcal{G}$. Under this formulation, an LM generates the label description $\hat{t}_\texttt{label}$ conditioned on $\mathcal{G}$, $\mathcal{G}_s$, and $t_\texttt{task}$.

Motivated by the ability to bridge the gap between pre-training and downstream data, we adopt the graph compression strategy for our injection framework. It fine-tunes an LLM encoder that compresses the input sequence $t_\texttt{enc}=[t_{\mathcal{G}_s}\|t_\texttt{task}]$ into a continuous graph prompt, where $\|$ denotes concatenation, $t_{\mathcal{G}_s}=\{t_v|v\in\mathcal{G}_s\}$ denotes the subgraph node sequence. A frozen LLM decoder then takes the learned prompt and the task description $t_\texttt{task}$ as input for label generation.

\begin{figure*}[htb]
\centering
\includegraphics[width=0.9\textwidth]{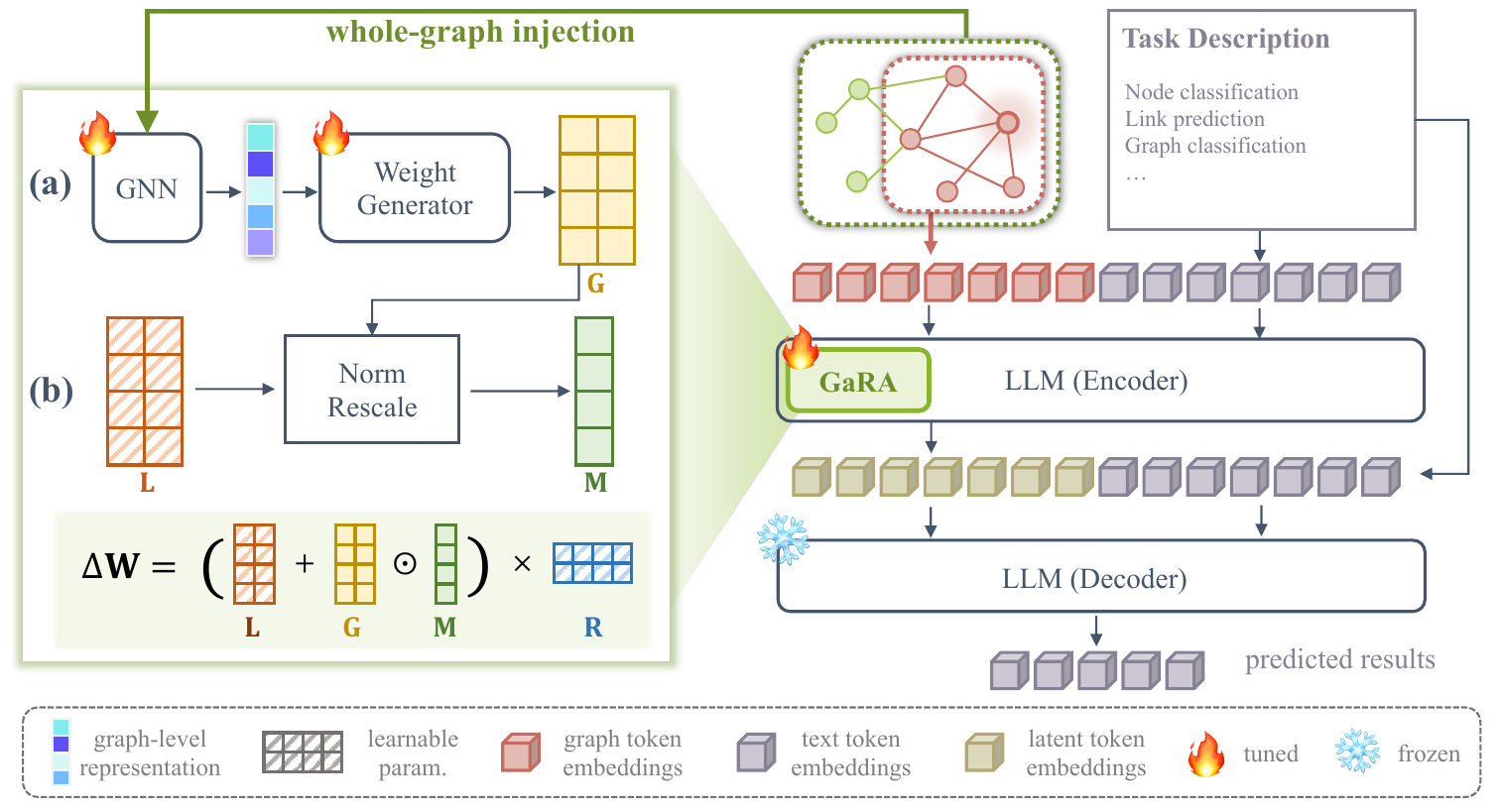}
\caption{\textbf{Pipeline of the Graph-aware LoRA Generation model, GaRA.} 
(a) To inject whole-graph information into LLMs, GaRA first applies GNN to the whole input graph, extracting a global representation via max pooling. This representation is then fed into an MLP-based weight generator to obtain low-rank weight updates $\mathbf{G}$. (b) To tackle the norm-amplification optimization bias, the norm of $\mathbf{G}$ is further rescaled based on the left low-rank matrix $\mathbf{L}$ from LoRA. The norm-rescaled $\mathbf{G}$ is combined with the low-rank matrices in LoRA to construct the final weight updates.}
\label{fig:pipeline}
\end{figure*}
\subsection{Injecting Graph Structures via Weight Generation}
\paragraph{General Formulation.}
Given the input feature matrix $\mathbf{X}$ and the adjacency matrix $\mathbf{A}$, we employ a GNN model to capture graph information in the whole-graph context
\begin{equation}\label{eq:gnn}
    \mathbf{h}=\texttt{MaxPool}(\texttt{GNN}(\mathbf{X},\mathbf{A})),
\end{equation}
where $\mathbf{h}\in\mathbb{R}^{d}$ denotes the resulting graph-level representation. In Eq.~\ref{eq:gnn}, the GNN learns node-level features via message passing on the input graph, encoding the graph topology implicitly. The subsequent $\texttt{MaxPool}$ operation lifts node features into a graph-level representation, providing whole-graph information for graph-aware weight generation.

Based on the graph-level representation $\mathbf{h}$, a non-linear mapping $\texttt{f}:\mathbb{R}^{d}\rightarrow\mathbb{R}^{p\times q}$ is adopted to generate the weight update, reformulating Eq.~\ref{eq:general} as
\begin{equation}\label{eq:gara-general}
    \mathbf{W} = \hat{\mathbf{W}} + \Delta \mathbf{W} = \hat{\mathbf{W}} + \texttt{f}(\mathbf{h}; \mathbf{\Theta}),
\end{equation}
where $\mathbf{\Theta}$ denotes the parameter matrix. To ensure the generated matrix covering both positive and negative values symmetrically, $\texttt{tanh}$ is employed in $\texttt{f}$.

\paragraph{Optimization Bias.} Graph-aware weight generation injects whole-graph information into the adaptation process without incurring a quadratic increase in computational cost. However, the primary formulation in Eq.~\ref{eq:gara-general} introduces an architectural bias during fine-tuning, under which the weight update is prone to norm amplification. Consequently, whenever increasing the norm of $\Delta\mathbf{W}$ decreases the training loss, gradient descent will try to amplify $\|\Delta \mathbf{W}\|$. This results in an architectural optimization bias toward norm-scaling weight update.

Formally, we show that for an adaptation strategy defined in Eq.~\ref{eq:gara-general}, there always exists a parameter update such that $\Delta\mathbf{W}$ can be scaled with its direction unchanged.
\begin{theorem}[Direction-preserving Update]\label{thrm:exist}
Let $\Delta\mathbf{W}=\texttt{f}(\mathbf{h}; \mathbf{\Theta})$. For a fixed input $\mathbf{h}$ and the Jacobian $J_\texttt{f}(\mathbf{\Theta})=\partial\texttt{f}(\mathbf{h}; \mathbf{\Theta})/\partial\mathbf{\Theta}$, there exists a nonzero parameter direction $\delta\mathbf{\Theta}$ such that under the update $\mathbf{\Theta}\leftarrow\mathbf{\Theta}+\epsilon\delta\mathbf{\Theta}$, the weight update satisfies $\Delta\mathbf{W}\leftarrow(1+\epsilon)\Delta\mathbf{W}+O(\epsilon)$.
\end{theorem}
Proof is provided in the Appendix~\ref{sec:app-proof}. Compared to the weight generation, the scaling of $\Delta\mathbf{W}$ in traditional fine-tuning methods arises solely from the scaling of the parameter matrices, which are discouraged by regularization techniques such as weight decay. However, when $\Delta\mathbf{W}$ is generated by the parameteric function $\texttt{f}(\mathbf{h}; \mathbf{\Theta})$, weight decay applied to $\mathbf{\Theta}$ does not directly constrain the norm of $\Delta\mathbf{W}$. Therefore, the direction-preserving norm amplification still emerges even under regularization. 

\subsection{GaRA: Graph-aware LoRA Generation}\label{ssec:gara}
\paragraph{LoRA-style Generation.}
Motivated by the efficiency and effectiveness of LoRA for model fine-tuning, we employ the low-rank strategy in our weight generation process (Fig.~\ref{fig:pipeline}). Specifically, take the left matrix generation as an example, the weight update in Eq.~\ref{eq:lora} can now be reformulated as
\begin{equation}\label{eq:gara}
    \mathbf{W} = \hat{\mathbf{W}} + (\mathbf{L} + \texttt{g}(\mathbf{h}; \mathbf{\Theta}))\mathbf{R}^\top,
\end{equation}
where $\texttt{g}:\mathbb{R}^{d}\rightarrow\mathbb{R}^{d_\texttt{M}\times r}$ generates the low-rank matrices $\mathbf{G}=\texttt{g}(\mathbf{h}; \mathbf{\Theta})$ for LoRA, $d_\texttt{M}=p$ when $\mathbf{G}$ serves as an augmented left low-rank matrix, and $d_\texttt{M}=q$ when it serves as an augmented right matrix. In practice, we combine linear mapping with $\texttt{tanh}$ to implement $\texttt{g}$. During fine-tuning, both $\mathbf{\Theta}$ and the low-rank matrices $\mathbf{L}$/$\mathbf{R}$ are trainable. Empirically, augmenting the left low-rank matrix generally achieves better adaptation results. Please refer to Sec.~\ref{sssec:abl-target} for more details. Therefore, in the following context, we stick to left matrix generation to formulate specific designs in GaRA.

The direction-preserving norm amplification problem also exists in the LoRA framework, as the results in Theorem~\ref{thrm:exist} directly extend to the low-rank setting. Proof is presented in the Appendix~\ref{sec:app-proof}.
\begin{corollary}[Direction-preserving Update in LoRA-style Generation]\label{thrm:exist-lora}
Let $\Delta\mathbf{W}=\mathbf{LR}^\top+\texttt{g}(\mathbf{h}; \mathbf{\Theta})\mathbf{R}^\top$. For a fixed input $\mathbf{h}$ and the Jacobian $J_\texttt{g}(\mathbf{\Theta})=\partial\texttt{g}(\mathbf{h}; \mathbf{\Theta})/\partial\mathbf{\Theta}$, there exists a nonzero parameter direction $\delta\mathbf{\Theta}$ such that under the update $\mathbf{\Theta}\leftarrow\mathbf{\Theta}+\epsilon\delta\mathbf{\Theta}$, the weight update satisfies $\Delta\mathbf{W}\leftarrow(1+\epsilon)\Delta\mathbf{W}+O(\epsilon)$.
\end{corollary}
These theoretical results can be verified in the empirical evaluation. During the training process, we observe that while LoRA exhibits a stable training dynamics, the primary LoRA-style generation method leads to continuous norm amplification on the generated low-rank matrix. Please refer to Sec.~\ref{ssec:norm} for more details. 

\paragraph{Influence of Norm Amplification.} Norm amplification creates an optimization shortcut during fine-tuning, which potentially leads to overfitting. Specifically, traditional LoRA uses weight decay to regularize the norm of weight updates, yielding much smaller values than the original weight. As a result, minimizing the loss typically requires optimizing both the update and the original weight. In contrast, the existence of direction-preserving updates in weight generation scales the norm of weight updates, weakening the contribution of the original weight. Therefore, the optimizer can solely rely on the generated weight updates to reduce the training loss, making it easier to fit the fine-tuning dataset and thus preferred during optimization. Such an optimization bias gradually dilutes the knowledge encoded in the original weight via large weight updates, leading to ability degradation for general tasks and model overfitting. This motivates us to introduce additional constraints in preventing direction-preserving norm amplification.

\paragraph{Constraining Norm Amplification.} To tackle the unconstrained norm amplification bias during optimization, we propose a simple yet effective solution by fixing the average column norm of the generated low-rank matrix $\mathbf{G}$ to match that of the corresponding left matrix $\mathbf{L}$. This choice is motivated by the fact that both $\mathbf{L}$ and $\mathbf{G}$ serve as the left factor to construct the weight update in the same manner. As a result, the norm of $\textbf{L}$ provides a natural and stable reference scale without introducing additional hyperparameters. With the norm rescaling, Eq.~\ref{eq:gara} can be reformulated as
\begin{equation}\label{eq:gara-f}
\begin{aligned}
    \mathbf{W} &= \hat{\mathbf{W}} + (\mathbf{L} + \mathbf{G}\odot\mathbf{M})\mathbf{R}^\top,\\
    \texttt{(GaRA)}\quad\mathbf{M} &= \frac{\texttt{avg}_j(\texttt{norm}(\mathbf{L}_{\cdot,j}))}{\texttt{avg}_j(\texttt{norm}(\mathbf{G}_{\cdot,j}))}\mathbf{1}^\top,\\
    \mathbf{G} &= \texttt{g}(\mathbf{h}; \mathbf{\Theta}),    
\end{aligned}
\end{equation}
where $\odot$ denotes element-wise multiplication, $\mathbf{M}$ denotes the rescaling matrix, $\texttt{norm}$ computes the $\texttt{L}_1$ norm of the input vector, and $\mathbf{1}\in\mathbf{R}^{r}$ denotes the all-one vector. The equations in Eq.~\ref{eq:gnn} and Eq.~\ref{eq:gara-f} are compiled as GaRA.

\begin{table*}
\caption{\textbf{Evaluation Results on Zero-shot Dataset Transfer (measured by accuracy for node classification tasks and AUC for graph classification tasks: \%).} ``NULL" indicates unsupported tasks.}
\label{tab:main-data}
\centering
\begin{sc}
\begin{small}
\begin{tabular}{lcccccccc} 
\hline
                     & \multicolumn{4}{c}{Node  Classification}                          & \multicolumn{2}{c}{Graph  Classification} & \multicolumn{2}{c}{Link Prediction}                                                                            \\ 
\cmidrule(lr){2-5}\cmidrule(lr){6-7}\cmidrule(r){8-9}
                     & Cora           & WikiCS         & Reddit         & Instagram      & BACE           & HIV                      & \begin{tabular}[c]{@{}c@{}}Amazon\\Photo\end{tabular} & \begin{tabular}[c]{@{}c@{}}Book\\History\end{tabular}  \\ 
\hline
Vicuna-7B             & 0.155          & 0.29           & 0.309          & 0.391          & 0.492          & 0.467                    & 0.576                                                 & 0.503                                                  \\
OFA                  & 0.189          & 0.361          & 0.498          & 0.580          & 0.483          & 0.404                    & 0.499                                                 & 0.457                                                  \\
GraphGPT             & 0.126          & -              & -              & -              & -              & -                        & -                                                     & -                                                      \\
LLaGA                & 0.156          & 0.601          & 0.499          & 0.397          & Null           & Null                     & 0.659                                                 & 0.602                                                  \\
TEA-GLM              & 0.202          & 0.449          & 0.491          & 0.479          & 0.467          & 0.498                    & 0.675                                                 & 0.584                                                  \\
GOFA                 & 0.039          & 0.613          & 0.493          & 0.367          & 0.500          & 0.481                    & 0.504                                                 & -                                                      \\
UniGTE               & 0.216          & 0.635          & 0.512          & 0.603          & 0.534          & 0.501                    & 0.681                                                 & 0.515                                                  \\ 
\hline
\textbf{GaRA (Ours)} & \textbf{0.251} & \textbf{0.669} & \textbf{0.518} & \textbf{0.611} & \textbf{0.549} & \textbf{0.512}           & \textbf{0.708}                                        & \textbf{0.722}                                         \\
\hline
\end{tabular}
\end{small}
\end{sc}
\end{table*}
\subsection{Model Training Strategy}
Following the common practice in graph compression strategy~\cite{chen_LLaGALargeLanguage_2024,wang_UniGTEUnifiedGraph_2025}, we adopt instruction tuning during the model training. Specifically, let $\mathcal{D}=\{(\mathcal{G}^{(i)},\mathcal{G}^{(i)}_s,t^{(i)}_{\mathcal{G}_s},t^{(i)}_\texttt{task},t^{(i)}_\texttt{label})\}_{i=1}^D$ be the training dataset. where $t^{(i)}_{\mathcal{G}_s}$ denotes the $i$-th subgraph node sequence, $t^{(i)}_\texttt{task}$ and $t^{(i)}_\texttt{label}$ denotes the corresponding task description and label, respectively. The model is required to maximize the log-likelihood function
\begin{equation}
\begin{aligned}
\Theta_\texttt{GaRA}^* &= \arg\max_{\Theta_\texttt{GaRA}}\sum_{i=1}^{D}\log P_{\Theta_\texttt{GaRA}}(t^{(i)}_{\texttt{label}}|I_i),\\
I_i &= (\mathcal{G}^{(i)},\mathcal{G}^{(i)}_s,t^{(i)}_{\mathcal{G}_s},t^{(i)}_\texttt{task}),
\end{aligned}
\end{equation}
where $P_{\Theta_\texttt{GaRA}}$ denotes the conditional probability distribution, $\Theta_\texttt{GaRA}$ denotes the set of trainable parameters during model fine-tuning. This target can be achieved by minimizing the cross-entropy loss
\begin{equation}
\mathcal{L}=-\sum_{i=1}^{D}\sum_{j=1}^{|t^{(i)}_{\texttt{label}}|}\log P_{\Theta_\texttt{GaRA}}(t^{(i)}_{\texttt{label},j}|t^{(i)}_{\texttt{label},<j},I_i), 
\end{equation}
where $t^{(i)}_{\texttt{label},<j}=(t^{(i)}_{\texttt{label},k}|k<j)$ denotes all the tokens that before the $j$-th token in the output sequence.

\section{Experiments}
In this section, we provide empirical evaluation results of GaRA on real-world benchmarks. GaRA\footnote{Codes are available at \url{https://github.com/sunjss/GaRA}.} is implemented with PyTorch~\cite{paszke_PyTorchImperativeStyle_2019} and PyTorch Geometric~\cite{fey_FastGraphRepresentation_2019}. The detailed experimental settings are presented in Appendix~\ref{sec:app-setup}.

\begin{table*}
\caption{\textbf{Evaluation Results in Tackling Norm Amplification (measured by accuracy for node classification tasks and AUC for graph classification tasks: \%).} ``Partial" denotes employing arXiv, BookChildren, AmazonComputer, and FB15K237 for tuning, where the first three datasets only involve node classification tasks. ``Full" denotes employing arXiv, BookChildren, AmazonComputer, FB15K237, and ChEMBL for tuning, where the first three datasets involve both node classification and link prediction tasks. ``NR." denotes norm rescaling. ``Rel. Imp." denotes relative improvement of employing norm rescaling in GaRA compared to no explicit norm constraints.}
\label{tab:abl-norm}
\centering
\begin{sc}
\begin{small}
\begin{tabular}{llcccccc} 
\hline
\multirow{2}{*}{Training Data} & \multirow{2}{*}{Model} & \multicolumn{2}{c}{Link  Prediction}                                                                          & \multicolumn{4}{c}{Node  Classification}  \\ 
\cmidrule(lr){3-3}\cmidrule(lr){4-4}\cmidrule(r){5-8}
                               &                        & \begin{tabular}[c]{@{}c@{}}Amazon\\Photo\end{tabular} & \begin{tabular}[c]{@{}c@{}}Book\\History\end{tabular} & Cora    & Reddit & WikiCS  & Instagram    \\ 
\hline
\multirow{3}{*}{Partial}       & GaRA (w/o nr)          & 0.527                                                 & 0.580                                                 & 0.203   & 0.502  & 0.611   & 0.569        \\
                               & GaRA (w nr)            & 0.529                                                 & 0.591                                                 & 0.208   & 0.513  & 0.664   & 0.607        \\
                               & Rel. Imp.              & 0.37\%                                                & 1.89\%                                               & 2.46\%  & 2.19\% & 8.67\%  & 6.67\%       \\
\multirow{3}{*}{Full}          & GaRA (w/o nr)          & 0.539                                                 & 0.581                                                 & 0.164   & 0.506  & 0.603   & 0.583        \\
                               & GaRA (w nr)            & 0.708                                                 & 0.722                                                 & 0.251   & 0.518  & 0.669   & 0.611        \\
                               & Rel. Imp.              & 31.35\%                                               & 24.26\%                                               & 53.04\% & 2.37\% & 10.94\% & 4.80\%       \\
\hline
\end{tabular}
\end{small}
\end{sc}
\end{table*}
\begin{figure*}[htb]
    \centering
    \subfigure[$\mathbf{G}$ and $\mathbf{L}$]{\label{fig:abl-norm-l}
      \begin{minipage}[t]{0.31\linewidth}
          \centering
          \includegraphics[width=\linewidth]{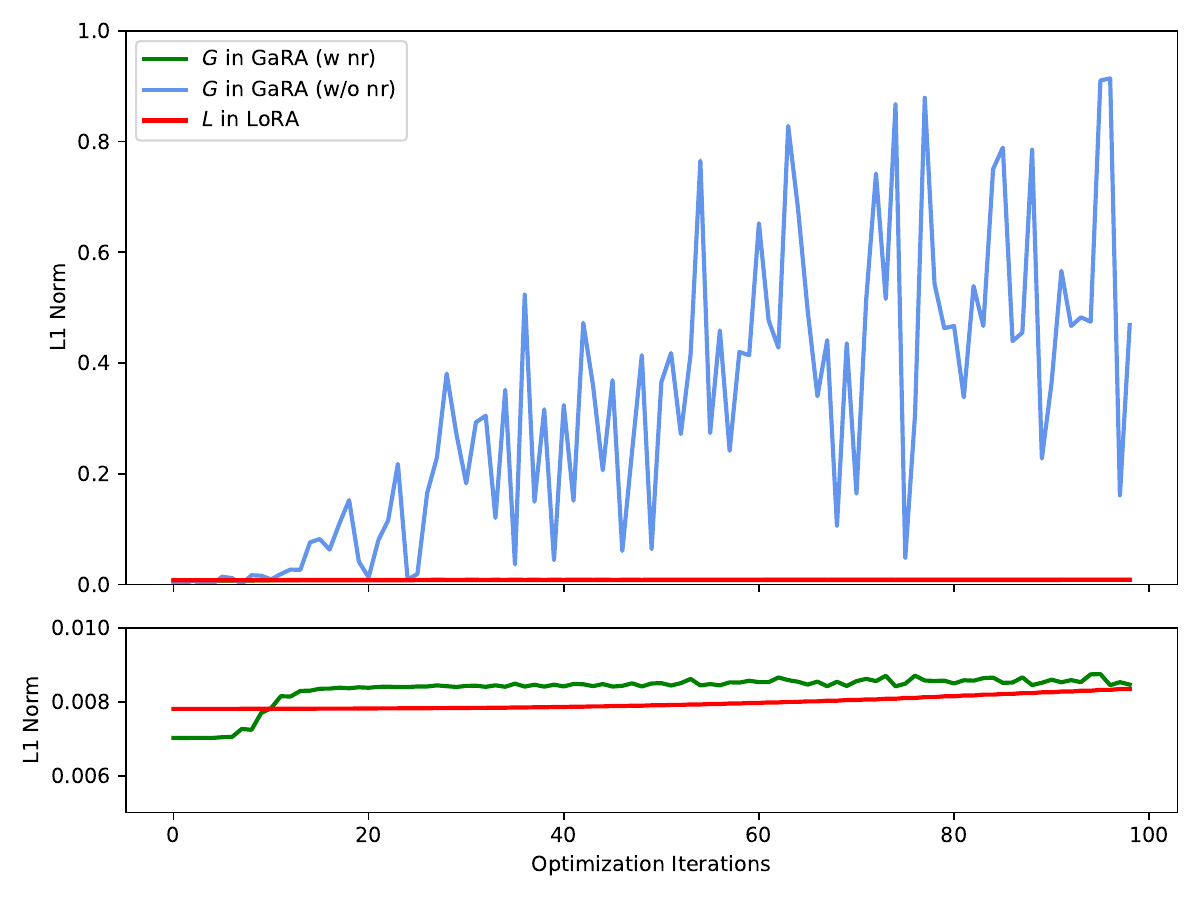}
      \end{minipage}}%
    \hfill
    \subfigure[$\Delta \mathbf{W}$ and $\hat{\mathbf{W}}$]{\label{fig:abl-norm-w}
          \begin{minipage}[t]{0.31\linewidth}
          \centering
          \includegraphics[width=\linewidth]{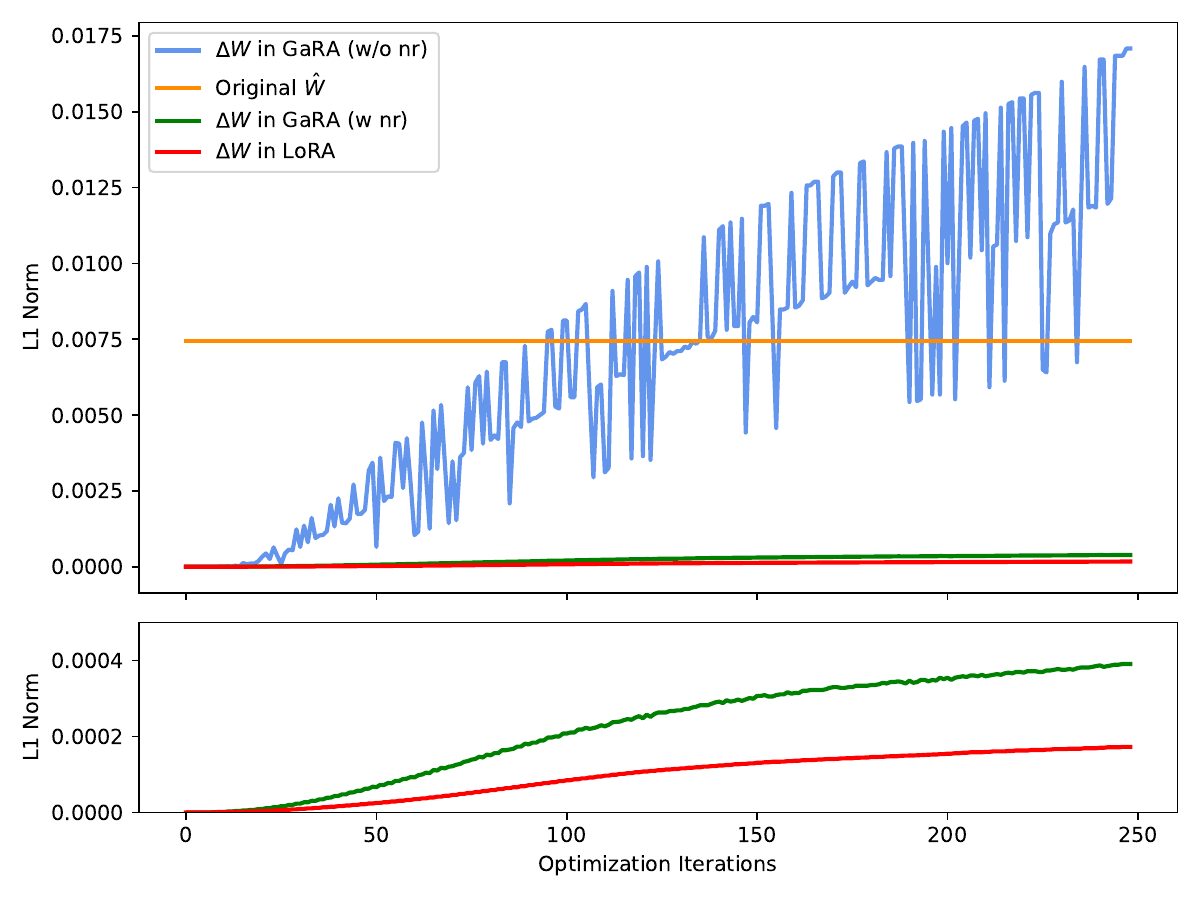}
          \end{minipage}}%
    \hfill
    \subfigure[Similarity and Test Loss]{\label{fig:abl-direction}
          \begin{minipage}[t]{0.32\linewidth}
          \centering
          \includegraphics[width=\linewidth]{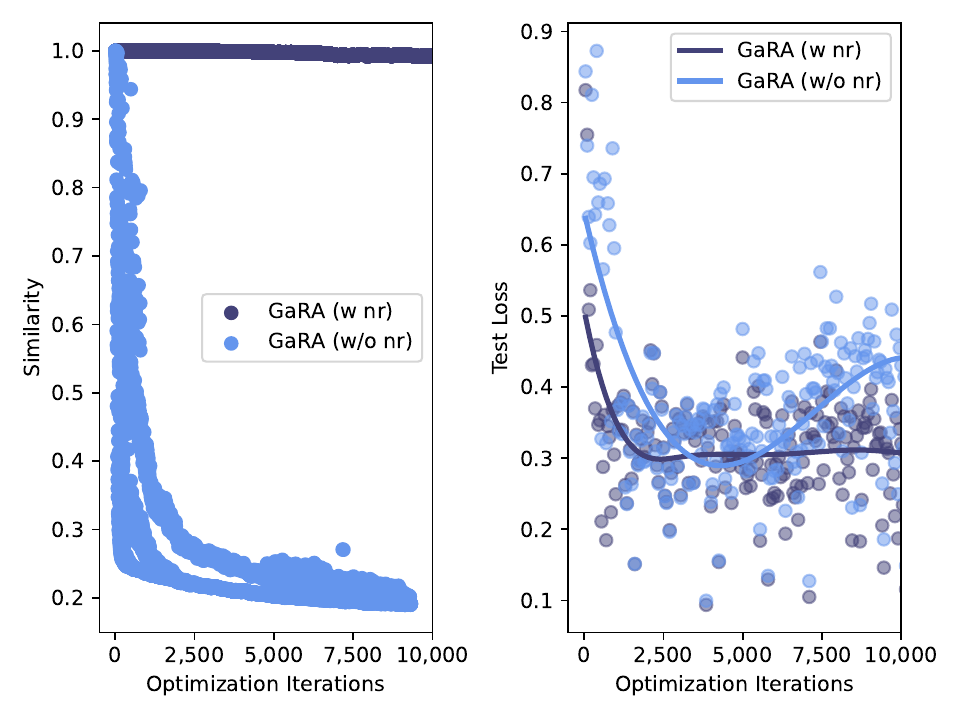}
          \end{minipage}}%
    \caption{\textbf{Effectiveness in Tackling Norm Amplification.} ``nr" denotes our norm rescaling strategy. (a) and (b) The lower figures are the zoom-in illustrations of the upper figures. (c) ``Similarity" denotes the cosine similarity between the updated weight $\mathbf{W}$ and the original weight $\hat{\mathbf{W}}$. Results are averaged over the 50 randomly selected test samples, where the selection is refreshed every 50 training steps.}
    \label{fig:abl-norm}
\end{figure*}
\subsection{Effectiveness Study on Zero-shot Dataset Transfer}
\paragraph{Experimental Setup.}
To evaluate the effectiveness of GaRA in generalizing to unseen datasets, we split training and inference sources into two non-overlapping sets of datasets. Specifically, GaRA and baselines are fine-tuned on five datasets: arXiv~\cite{hu_OpenGraphBenchmark_2020}, BookChildren~\cite{yan_ComprehensiveStudyTextattributed_2023}, AmazonComputer~\cite{shchur_PitfallsGraphNeural_2019}, FB15K237~\cite{liu_OneAllTraining_2023}, and ChEMBL~\cite{chen_TextspaceGraphFoundation_2024}, encompassing node classification, link prediction, and graph classification tasks. On arXiv, BookChildren, and AmazonComputer, both node classification and link prediction tasks are included. After fine-tuning, GaRA and baseline models are evaluated in a zero-shot setting on unseen datasets, including (Cora~\cite{wen_AugmentingLowResourceText_2023}, WikiCS~\cite{mernyei_WikiCSWikipediaBasedBenchmark_2022}, Reddit~\cite{li_GLBenchComprehensiveBenchmark_2024}, and Instagram~\cite{li_GLBenchComprehensiveBenchmark_2024}) for node classification, (AmazonPhoto and BookHistory)~\cite{yan_ComprehensiveStudyTextattributed_2023} for link prediction, and (ChemHIV and ChemBACE)~\cite{wu_MoleculeNetBenchmarkMolecular_2018} for graph classification. Dataset statistics are summarized in Tab.~\ref{tab:statistics}.

Baseline models include the output-level injection method OFA~\cite{liu_OneAllTraining_2023}, input-level injection methods (GraphGPT~\cite{tang_GraphGPTGraphInstruction_2024a}, LLaGA~\cite{chen_LLaGALargeLanguage_2024}, TEA-GLM~\cite{wang_LLMsZeroshotGraph_2024a}, GOFA~\cite{kong_GOFAGenerativeOneAll_2024}, and UniGTE~\cite{wang_UniGTEUnifiedGraph_2025}), and a pre-trained LLM Vicuna-7B~\cite{zheng_JudgingLLMJudgeMTBench_2023}. GaRA is implemented with GCN~\cite{kipf_SemiSupervisedClassificationGraph_2017} and $\texttt{MaxPooling}$ for whole-graph information extraction. In practice, the graph-aware weight generation is only applied to the $21$-st layer. Please refer to Sec.~\ref{ssec:abl} for more ablation details.

\paragraph{Performance.}
Tab.~\ref{tab:main-data} shows the comparison results between GaRA and baseline models for zero-shot dataset transfer, where GaRA consistently achieves better performance across various datasets and task types. For node classification and link prediction tasks, methods that solely depend on the input-level injection can only encode subgraph information, where the whole-graph information is discarded. In contrast, GaRA further performs weight-level injection by explicitly incorporating whole-graph information, encoding both local features at the input level and higher-level features at the weight level. This design avoids potential information loss in the discarded whole graph, resulting in consistently better performance of GaRA on both node and link tasks. Moreover, on graph classification tasks, where whole-graph information is directly provided to LLMs, GaRA also surpasses baseline models. We attribute this advantage to the explicit injection of whole-graph representations, whereas input-injection methods must implicitly infer global structure through pairwise node interactions, which can be less effective for capturing graph representations.

\begin{figure*}[htb]
    \centering
    \begin{minipage}[t]{0.46\linewidth}
          \centering
          \includegraphics[width=0.7\linewidth]{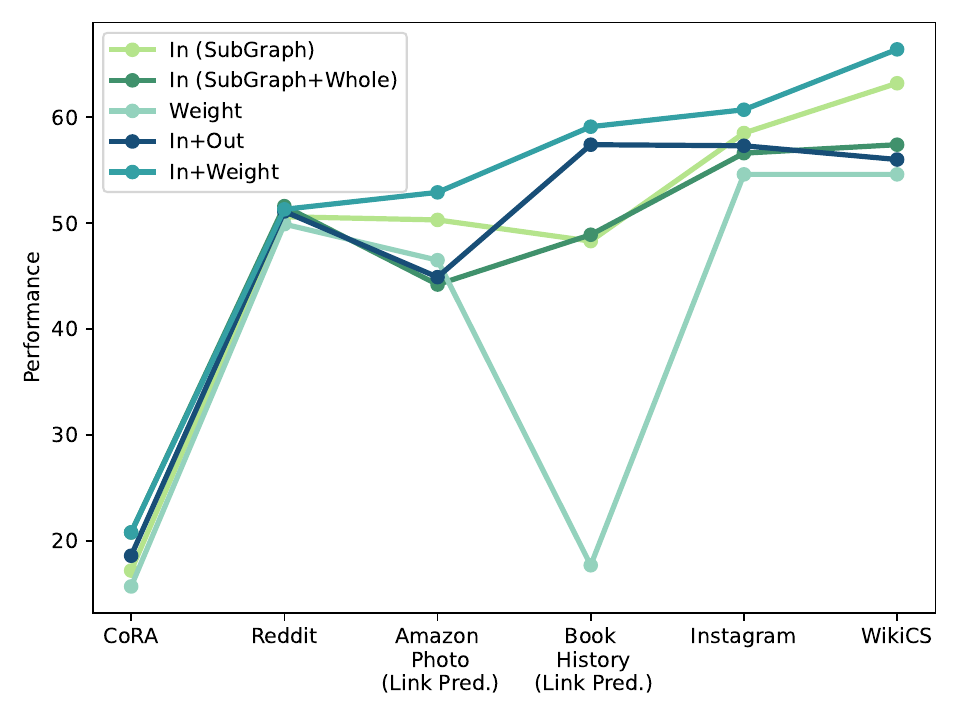}
    \caption{\textbf{Comparison Among Injection Paradigms.} ``SubGraph" solely employs subgraph features as input. ``Whole" further concatenates global features to the input sequences.}
    \label{fig:abl-paradigm}
    \end{minipage}
    \hfill
    \begin{minipage}[t]{0.46\linewidth}
          \centering
          \includegraphics[width=0.75\linewidth]{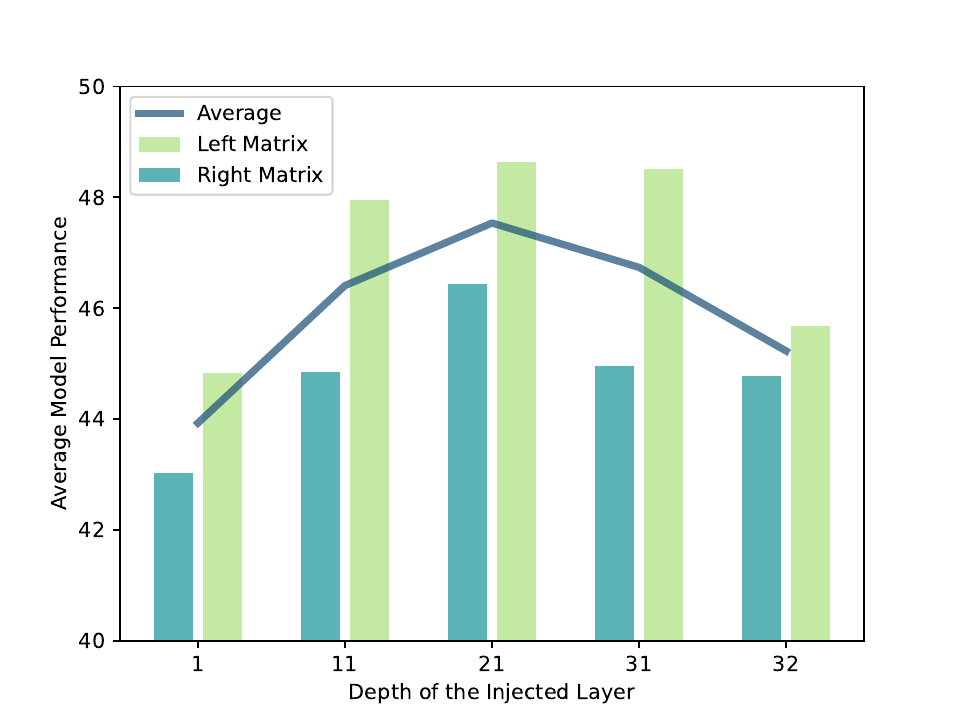}
        \caption{\textbf{Impact of the Injection Position and Generation Target.} ``Left/Right" denotes that the generated low-rank matrix serves as an augmented left or right matrix.}
        \label{fig:abl-pos-target}
    \end{minipage}
\end{figure*}
\subsection{Effectiveness in Tackling Norm Amplification}\label{ssec:norm}
To study the effectiveness of GaRA in tackling norm amplification, we start by presenting the norm values of the generated low-rank matrix with or without norm rescaling in Fig.~\ref{fig:abl-norm-l} and \ref{fig:abl-norm-w}. Then we dive into the potential influence of norm amplification. Based on the analysis in Section~\ref{ssec:gara}, norm amplification may lead to dilution of the original weight and model overfitting. Here, Fig.~\ref{fig:abl-direction} presents the cosine similarity between the updated weight $\mathbf{W}$ and the original weight $\hat{\mathbf{W}}$ for weight dilution and the test loss for model overfitting. The results are from the value matrix projection in the $21$-st layer of the LLM encoder Vicuna-7B. 

As shown in Fig.~\ref{fig:abl-norm-l}, without norm rescaling, the fine-tuning process exhibits continuous amplification on the norm of the generated low-rank matrix. By employing the generated low-rank matrix to compute the weight update $\Delta \mathbf{W}$, either the norm or the direction of the updated weight $\mathbf{W}$ is dominated by the update part over the original weight $\hat{\mathbf{W}}$ (Fig.~\ref{fig:abl-norm-w} and Fig.~\ref{fig:abl-direction} left). As a result, the model without norm rescaling is prone to overfitting (Fig.~\ref{fig:abl-direction} right), leading to degraded performance on transfer tasks. In contrast, GaRA with norm rescaling better preserves the knowledge in the original weights, while only partially adjusting the final weights through the generated updates.

Besides the analysis on model characteristics, Tab.~\ref{tab:abl-norm} further presents the performance comparison with and without norm rescaling. It shows that weight generation without norm rescaling achieves inferior performance compared to generating with norm rescaling. When the amount of training data increases, norm rescaling empowers GaRA to gain enhanced transferability, while generating without norm rescaling gains limited improvements. Therefore, the performance gap between the two methods becomes larger. All these results indicate the negative impact of the norm amplification during optimization and the effectiveness of our norm-rescaling strategy in tackling this architectural optimization bias. Except for norm rescaling, there are other potential strategies for constraining norm amplification. Please refer to Appendix~\ref{ssec:app-abl-norm} for more details.

\subsection{Ablation Study}\label{ssec:abl}
In this section, we conduct ablation studies on the architectural designs in GaRA, including comparing GaRA with other injection solutions, the position choice for weight-level graph information injection, and different generating inputs and targets. Model variants are fine-tuned on arXiv, BookChildren, AmazonComputer, and FB15K237, where arXiv, BookChildren, and AmazonComputer contain only node classification tasks.

\subsubsection{Impact of the Injection Paradigm}
GaRA combines input-level and weight-level injection to provide graph structures for LM backbones. To isolate the contribution of our weight-level injection, we compare parameter-matching models under the same UniGTE framework with subgraph input injection, whole-graph input injection, weight injection, input-weight injection (GaRA), and input-output injection. Specifically, subgraph input injection is based on UniGTE, while whole-graph input injection directly concatenates whole-graph information learned by GNN at the input level. Weight injection gains graph information solely from the weight generator by masking the graph embedding in the LM input. Input-output injection applies a GNN to the whole graph to extract node-level representations and add them to the LLM encoder output. Results are presented in Fig.~\ref{fig:abl-paradigm}, following the ``Partial" setup in Tab.~\ref{tab:abl-norm}. We can see that combining input and weight injection consistently achieves better performance than other variants, indicating that the two injection paradigms provide complementary benefits.

\begin{figure*}[htb]
    \centering
    \begin{minipage}[t]{0.46\linewidth}
          \centering
          \includegraphics[width=0.75\linewidth]{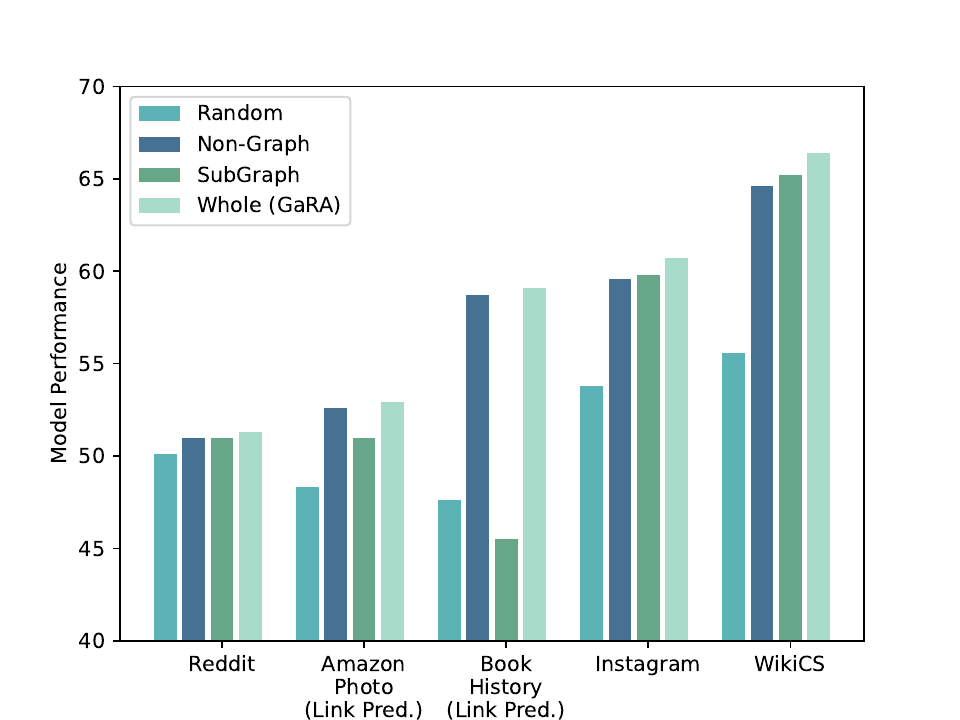}
    \caption{\textbf{Comparison Among Generation Inputs.} ``Random", ``Non-Graph", and "SubGraph" denote generating low-rank matrices from random vectors, embedded task description, and GNN-processed subgraph features, respectively.}
    \label{fig:abl-input}
    \end{minipage}
    \hfill
    \begin{minipage}[t]{0.46\linewidth}
          \centering
          \includegraphics[width=0.75\linewidth]{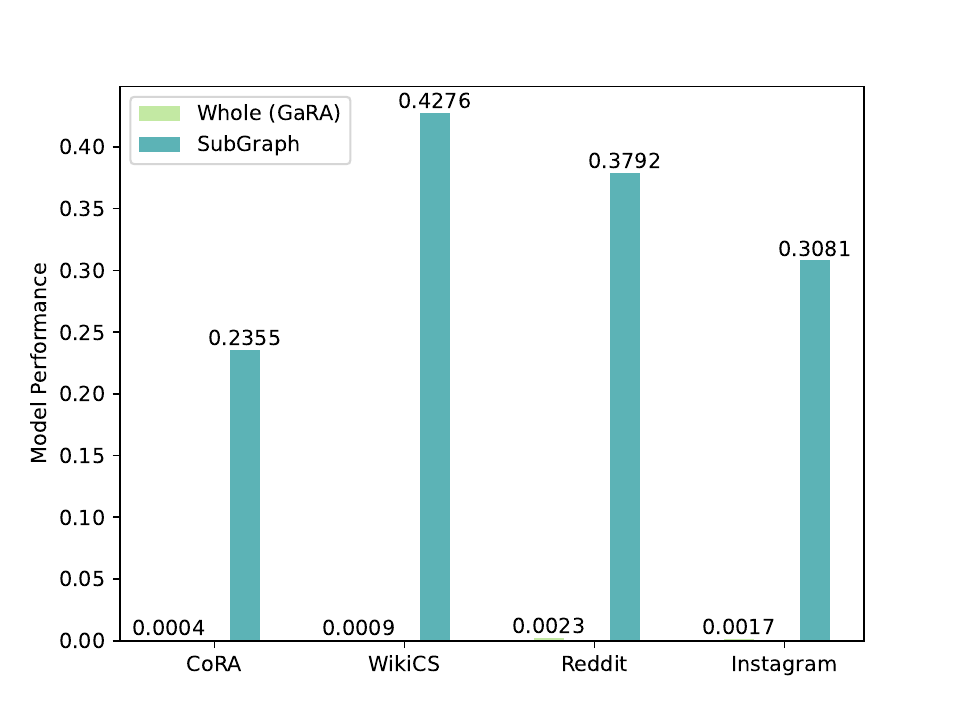}
        \caption{\textbf{Redundancy Analysis with Similarity Comparison.} Cosine similarity is computed between hidden representations obtained with the vanilla LoRA versus GaRA with ``Whole" graph/``SubGraph" as generator inputs.}
        \label{fig:abl-sim}
    \end{minipage}
\end{figure*}
\subsubsection{Impact of the Weight Injecting Position}\label{sssec:abl-pos} 
In practice, GaRA applies weight-level injection to only a single layer of the backbone LLM. To study the impact of different injection positions, we adopt Vicuna-7B as the backbone model, which consists of 32 layers. Injection positions are selected from $\{1, 11, 21, 31, 32\}$-th layer. Fig.~\ref{fig:abl-pos-target} reports the average performance across the node classification and link prediction datasets from Tab.~\ref{tab:main-data}. We can see that injecting at the 21-st layer achieves the best average performance. As the injection position moves from shallow to intermediate layers, performance gradually improves and then degrades near the output layer. This is because injecting graph-aware features into deeper layers yields a more direct influence on the final predictions. However, the recovery trend does not extend to the final layer, as such late-stage perturbations bypass most of the backbone’s forward transformations and disrupt the final output representations.

\subsubsection{Impact of the Generation Target}\label{sssec:abl-target}
GaRA adopts low-rank fine-tuning to adapt pre-trained LLMs to graph tasks. During the generation process, either the left or right low-rank matrix can serve as the weight-level injection target. To investigate the impact of the two generation targets, we conduct an ablation study by employing the generated low-rank matrix $\mathbf{G}$ as an augmented left or right low-rank matrix. Following the same experimental setup in Sec~\ref{sssec:abl-pos}, Fig.~\ref{fig:abl-pos-target} compares the average performance with the two generation targets. Results show that generating the augmented left low-rank matrix consistently outperforms the augmented right matrix. 
We also note that either generating the left or the right matrix for layers at different depths follows the same performance trend in Sec.~\ref{sssec:abl-pos}, further verifying our conclusion.

\subsubsection{Impact of the Generation Input}\label{sssec:abl-input}
GaRA employs the whole-graph representation to generate weight updates. To validate the benefit of whole-graph information, we compare different generation inputs, including random vectors, non-graph representations from the embedded task description, subgraph representations, and whole-graph representations (GaRA). We also compare GaRA with an MLP variant to isolate the contribution of graph-level representations. Please refer to Appendix~\ref{ssec:app-backbone} for details. Results in Fig.~\ref{fig:abl-input} show that GaRA consistently achieves better performance than other variants. Specifically, random vectors cannot provide any informative features and result in the worst performance. Non-graph and subgraph representations achieve similar performance, where discarding whole-graph information leads to inferior performance. This stems from the fact that global features are complementary to the input-level local features, while adding more local features shows redundancy. 

To further demonstrate such redundancy, we compute the cosine similarity between hidden representations obtained with pure LoRA versus the variant with subgraph representations and GaRA. 50 samples are randomly selected from the test sets. Averaged results in Fig.~\ref{fig:abl-sim} show that generating with local information offers similar features to the original features, indicating redundancy. In contrast, generating with global information offers less similar features to the original ones, indicating complementarity. All these results demonstrate the importance of further complementing whole-graph information for the primary input injection.

\begin{table}
\caption{\textbf{Cost Analysis (s/sample).}}
\label{tab:time}
\centering
\begin{sc}
\begin{small}
\begin{tabular}{lcccc} 
\hline
          & UniGTE & GaRA & LLaGA & Vicuna-7B          \\ 
\hline
Train     & 1.28   & 1.46 & 0.83  & -  \\
Inference & 0.67   & 0.71 & 0.12  & 1.32              \\
\hline
\end{tabular}
\vskip -0.2in
\end{small}
\end{sc}
\end{table}
\subsection{Cost Analysis}
GaRA depends on a GNN model to learn global features and a single-layer perceptron to generate the low-rank weight matrix. In practice, we employ GCN~\cite{kipf_SemiSupervisedClassificationGraph_2017} as the backbone to construct our GNN model, which requires a computational complexity of $O(m)+O(n)$. Therefore, the computational bottleneck mainly comes from the LLM backbone instead of our external weight generation module. To empirically evaluate the efficiency of GaRA, we conduct a cost analysis during both the tuning and the inference processes. Specifically, we employ (arXiv, BookChildren, AmazonComputer, and FB15K237) as the training datasets, and (Cora, BookHistory, and AmazonPhoto) as the inference datasets, computing the average cost per sample. Results in Tab.~\ref{tab:time} show that GaRA stays on par with other LLM-based models.

\section{Conclusion}
In this paper, we proposed a novel weight-level injection paradigm for adapting LMs to graph tasks. Unlike prior input- or output-level methods, which either overlook higher-level information from the whole-graph context or struggle to generalize across tasks, our paradigm enables explicit whole-graph injection by generating weight updates during fine-tuning and applying them to hidden-layer representations. In this way, weight-level injection complements subgraph information introduced at the input level and mitigates potential information loss. Building on this paradigm, we then proposed a graph-aware weight generation method, named GaRA. GaRA consists of a LoRA-style weight generation module and a norm-rescaling module. The weight generation module produces weight updates that inject whole-graph information without incurring a quadratic increase in computational cost. The norm-rescaling module addresses a direction-preserving norm amplification issue, which we identify and theoretically characterize as an architectural optimization bias in primary weight generation. Empirical results demonstrate that GaRA consistently outperforms existing baselines on zero-shot transfer tasks and effectively alleviates norm amplification, validating both the effectiveness and robustness of the proposed approach. For the limitation discussion, please refer to Appendix~\ref{sec:app-limitation}.

\section*{Acknowledgement}
This work was supported in part by the National Key R\&D Program of China under Grant 2023YFC2508704,  in part by the National Natural Science Foundation of China under grant number 62236008, in part by the Natural Science Foundation of Beijing under grant number L251082, and in part by Shandong Provincial Natural Science Foundation under project ZR2025ZD01. The authors would like to thank the anonymous reviewers for their helpful comments and suggestions that improved this manuscript.

\section*{Impact Statement}

This paper presents work whose goal is to advance the field of Machine Learning. There are many potential societal consequences of our work, none of which we feel must be specifically highlighted here.


\nocite{langley00}

\bibliography{example_paper}

@inproceedings{reimers-2019-sentence-bert,
    title = "Sentence-BERT: Sentence Embeddings using Siamese BERT-Networks",
    author = "Reimers, Nils and Gurevych, Iryna",
    booktitle = "Proceedings of the 2019 Conference on Empirical Methods in Natural Language Processing",
    month = "11",
    year = "2019",
    publisher = "Association for Computational Linguistics",
}

@misc{llama31_IntroducingLlama31_2024,
  title = {Introducing {{Llama}} 3.1: {{Our}} Most Capable Models to Date},
  shorttitle = {Introducing {{Llama}} 3.1},
  year = 2024,
  journal = {Meta AI},
  urldate = {2026-05-20},
  howpublished = {https://ai.meta.com/blog/meta-llama-3-1/},
  langid = {english}
}

@inproceedings{ma_VisualPerceptionLarge_2024,
  title = {Visual Perception by Large Language Model's Weights},
  booktitle = {Proceedings of the 38th {{International Conference}} on {{Neural Information Processing Systems}}},
  author = {Ma, Feipeng and Xue, Hongwei and Zhou, Yizhou and Wang, Guangting and Rao, Fengyun and Yan, Shilin and Zhang, Yueyi and Wu, Siying and Shou, Mike Zheng and Sun, Xiaoyan},
  year = 2024,
  volume = {37},
  pages = {28615--28635},
  publisher = {Curran Associates Inc.},
  address = {Red Hook, NY, USA},
  urldate = {2026-05-19}
}

@inproceedings{zhao_FullyinductiveNodeClassification_2024,
  title = {Fully-Inductive {{Node Classification}} on {{Arbitrary Graphs}}},
  booktitle = {The {{Thirteenth International Conference}} on {{Learning Representations}}},
  author = {Zhao, Jianan and Zhu, Zhaocheng and Galkin, Mikhail and Mostafa, Hesham and Bronstein, Michael M. and Tang, Jian},
  year = {2024},
  urldate = {2025-11-19},
  langid = {english}
}

@article{bessadok_GraphNeuralNetworks_2023,
  title = {Graph {{Neural Networks}} in {{Network Neuroscience}}},
  author = {Bessadok, Alaa and Mahjoub, Mohamed Ali and Rekik, Islem},
  year = {2023},
  journal = {IEEE Transactions on Pattern Analysis and Machine Intelligence},
  volume = {45},
  number = {5},
  pages = {5833--5848},
  issn = {1939-3539},
  urldate = {2023-11-05}
}

@inproceedings{paszke_PyTorchImperativeStyle_2019,
  title = {{{PyTorch}}: An Imperative Style, High-Performance Deep Learning Library},
  shorttitle = {{{PyTorch}}},
  booktitle = {International {{Conference}} on {{Neural Information Processing Systems}}},
  author = {Paszke, Adam and Gross, Sam and Massa, Francisco and Lerer, Adam and Bradbury, James and Chanan, Gregory and Killeen, Trevor and Lin, Zeming and Gimelshein, Natalia and Antiga, Luca and Desmaison, Alban and K{\"o}pf, Andreas and Yang, Edward and DeVito, Zach and Raison, Martin and Tejani, Alykhan and Chilamkurthy, Sasank and Steiner, Benoit and Fang, Lu and Bai, Junjie and Chintala, Soumith},
  year = {2019},
  pages = {8026--8037},
  publisher = {{Curran Associates Inc.}},
  address = {{Red Hook, NY, USA}}
}

@inproceedings{hamilton_InductiveRepresentationLearning_2017,
  title = {Inductive Representation Learning on Large Graphs},
  booktitle = {International {{Conference}} on {{Neural Information Processing Systems}}},
  author = {Hamilton, William L. and Ying, Rex and Leskovec, Jure},
  year = {2017},
  pages = {1025--1035},
  publisher = {{Curran Associates Inc.}},
  address = {{Red Hook, USA}}
}

@inproceedings{hu_OpenGraphBenchmark_2020,
  title = {Open {{Graph Benchmark}}: {{Datasets}} for {{Machine Learning}} on {{Graphs}}},
  shorttitle = {Open {{Graph Benchmark}}},
  booktitle = {Advances in {{Neural Information Processing Systems}}},
  author = {Hu, Weihua and Fey, Matthias and Zitnik, Marinka and Dong, Yuxiao and Ren, Hongyu and Liu, Bowen and Catasta, Michele and Leskovec, Jure},
  year = {2020},
  volume = {33},
  pages = {22118--22133},
  publisher = {{Curran Associates, Inc.}},
  urldate = {2024-01-28}
}

@inproceedings{fey_FastGraphRepresentation_2019,
  title = {Fast {{Graph Representation Learning}} with {{PyTorch Geometric}}},
  booktitle = {International {{Conference}} on {{Learning Representations Workshop}} on {{Graphs}} and {{Manifolds}}},
  author = {Fey, Matthias and Lenssen, Jan Eric},
  year = {2019},
  langid = {english}
}

@inproceedings{kipf_SemiSupervisedClassificationGraph_2017,
  title = {Semi-{{Supervised Classification}} with {{Graph Convolutional Networks}}},
  booktitle = {International {{Conference}} on {{Learning Representations}}},
  author = {Kipf, Thomas N. and Welling, Max},
  year = {2017},
  eprint = {1609.02907},
  address = {{Toulon, France}},
  archiveprefix = {arxiv},
  langid = {english}
}

@article{li_DeepGCNsMakingGCNs_2021,
  title = {{{DeepGCNs}}: {{Making GCNs Go}} as {{Deep}} as {{CNNs}}},
  shorttitle = {{{DeepGCNs}}},
  author = {Li, Guohao and M{\"u}ller, Matthias and Qian, Guocheng and Delgadillo, Itzel C. and Abualshour, Abdulellah and Thabet, Ali and Ghanem, Bernard},
  year = {2021},
  journal = {IEEE Transactions on Pattern Analysis and Machine Intelligence},
  eprint = {1910.06849},
  pages = {1--1},
  issn = {0162-8828, 2160-9292, 1939-3539},
  doi = {10/gk3x7j},
  urldate = {2021-07-03},
  archiveprefix = {arxiv},
  langid = {english}
}

@inproceedings{lim_LargeScaleLearning_2021,
  title = {Large {{Scale Learning}} on {{Non-Homophilous Graphs}}: {{New Benchmarks}} and {{Strong Simple Methods}}},
  shorttitle = {Large {{Scale Learning}} on {{Non-Homophilous Graphs}}},
  booktitle = {Advances in {{Neural Information Processing Systems}}},
  author = {Lim, Derek and Hohne, Felix and Li, Xiuyu and Huang, Sijia Linda and Gupta, Vaishnavi and Bhalerao, Omkar and Lim, Ser Nam},
  year = {2021},
  volume = {34},
  pages = {20887--20902},
  publisher = {{Curran Associates, Inc.}},
  urldate = {2024-01-06}
}

@article{t-SNE,
  title={Visualizing data using t-SNE.},
  author={Van der Maaten, Laurens and Hinton, Geoffrey},
  journal={Journal of machine learning research},
  volume={9},
  number={11},
  year={2008}
}

@inproceedings{morris_TUDatasetCollectionBenchmark_2020,
  title = {{{TUDataset}}: {{A}} Collection of Benchmark Datasets for Learning with Graphs},
  shorttitle = {{{TUDataset}}},
  booktitle = {International {{Conference}} on {{Machine Learning Workshop}} on {{Graph Representation Learning}} and {{Beyond}}},
  author = {Morris, Christopher and Kriege, Nils M. and Bause, Franka and Kersting, Kristian and Mutzel, Petra and Neumann, Marion},
  year = {2020},
  eprint = {2007.08663},
  primaryclass = {cs, stat},
  publisher = {{arXiv}},
  address = {{Virtual Only}},
  urldate = {2024-01-28},
  archiveprefix = {arxiv}
}

@article{shchur_PitfallsGraphNeural_2019,
  title = {Pitfalls of {{Graph Neural Network Evaluation}}},
  author = {Shchur, Oleksandr and Mumme, Maximilian and Bojchevski, Aleksandar and G{\"u}nnemann, Stephan},
  year = {2019},
  journal = {arXiv:1811.05868 [cs, stat]},
  eprint = {1811.05868},
  primaryclass = {cs, stat},
  urldate = {2021-07-03},
  archiveprefix = {arxiv},
  langid = {english}
}

@inproceedings{sun_RelievingAggregatingEffect_2025,
  title = {Relieving the {{Over-Aggregating Effect}} in {{Graph Transformers}}},
  booktitle = {The {{Thirty-ninth Annual Conference}} on {{Neural Information Processing Systems}}},
  author = {Sun, Junshu and Chang, Wanxing and Yang, Chenxue and Huang, Qingming and Wang, Shuhui},
  year = {2025},
  urldate = {2025-11-08},
  langid = {english}
}

@inproceedings{xu*_HowPowerfulAre_2019,
  title = {How {{Powerful}} Are {{Graph Neural Networks}}?},
  booktitle = {International {{Conference}} on {{Learning Representations}}},
  author = {Xu, Keyulu and Hu, Weihua and Leskovec, Jure and Jegelka, Stefanie},
  year = {2019},
  address = {{New Orleans, LA, USA}},
  urldate = {2022-08-05},
  langid = {english}
}

@inproceedings{sun_AllinARow_2023,
  title = {All in a Row: {{Compressed}} Convolution Networks for Graphs},
  booktitle = {International {{Conference}} on {{Machine Learning}}},
  author = {Sun, Junshu and Wang, Shuhui and Han, Xinzhe and Xue, Zhe and Huang, Qingming},
  year = {2023},
  volume = {202},
  pages = {33061--33076},
  publisher = {{PMLR}},
  address = {{Honolulu, USA}}
}

@inproceedings{ying_TransformersReallyPerform_2021,
  title = {Do {{Transformers Really Perform Bad}} for {{Graph Representation}}?},
  booktitle = {Advances in {{Neural Information Processing Systems}}},
  author = {Ying, Chengxuan and Cai, Tianle and Luo, Shengjie and Zheng, Shuxin and Ke, Guolin and He, Di and Shen, Yanming and Liu, Tie-Yan},
  year = {2021},
  volume = {34},
  eprint = {2106.05234},
  pages = {28877--28888},
  publisher = {{Curran Associates, Inc.}},
  urldate = {2021-12-22},
  archiveprefix = {arxiv}
}

@article{zhang_DynamicGraphMessage_2022,
  title = {Dynamic {{Graph Message Passing Networks}}},
  author = {Zhang, Li and Chen, Mohan and Arnab, Anurag and Xue, Xiangyang and Torr, Philip H. S.},
  year = {2022},
  journal = {IEEE Transactions on Pattern Analysis and Machine Intelligence},
  pages = {5712--5730},
  urldate = {2022-11-14},
  langid = {english}
}

@inproceedings{zhang_LinkPredictionBased_2018,
  title = {Link {{Prediction Based}} on {{Graph Neural Networks}}},
  booktitle = {International {{Conference}} on {{Neural Information Processing Systems}}},
  author = {Zhang, Muhan and Chen, Yixin},
  year = {2018},
  pages = {5171--5181},
  publisher = {{Curran Associates Inc.}},
  address = {{Red Hook, NY, USA}},
  urldate = {2022-08-09}
}

@misc{openai_GPT4oSystemCard_2024,
  title = {{{GPT-4o System Card}}},
  author = {OpenAI and Hurst, Aaron and Lerer, Adam and Goucher, Adam P. and Perelman, Adam and Ramesh, Aditya and Clark, Aidan and Ostrow, A. J. and Welihinda, Akila and Hayes, Alan and Radford, Alec and M{\k a}dry, Aleksander and {Baker-Whitcomb}, Alex and Beutel, Alex et al.},
  year = {2024},
  number = {arXiv:2410.21276},
  eprint = {2410.21276},
  primaryclass = {cs},
  publisher = {arXiv},
  archiveprefix = {arXiv}
}

@inproceedings{li_ZeroGInvestigatingCrossdataset_2024,
  title = {{{ZeroG}}: {{Investigating Cross-dataset Zero-shot Transferability}} in {{Graphs}}},
  shorttitle = {{{ZeroG}}},
  booktitle = {{{ACM SIGKDD Conference}} on {{Knowledge Discovery}} and {{Data Mining}}},
  author = {Li, Yuhan and Wang, Peisong and Li, Zhixun and Yu, Jeffrey Xu and Li, Jia},
  year = {2024},
  pages = {1725--1735},
  publisher = {Association for Computing Machinery},
  address = {New York, NY, USA}
}

@inproceedings{liu_OneAllTraining_2023,
  title = {One {{For All}}: {{Towards Training One Graph Model For All Classification Tasks}}},
  shorttitle = {One {{For All}}},
  booktitle = {International {{Conference}} on {{Learning Representations}}},
  author = {Liu, Hao and Feng, Jiarui and Kong, Lecheng and Liang, Ningyue and Tao, Dacheng and Chen, Yixin and Zhang, Muhan},
  year = {2023},
}

@inproceedings{wang_CanLanguageModels_2023,
  title = {Can Language Models Solve Graph Problems in Natural Language?},
  booktitle = {International {{Conference}} on {{Neural Information Processing Systems}}},
  author = {Wang, Heng and Feng, Shangbin and He, Tianxing and Tan, Zhaoxuan and Han, Xiaochuang and Tsvetkov, Yulia},
  year = {2023},
  pages = {30840--30861},
  publisher = {Curran Associates Inc.},
  address = {Red Hook, NY, USA},
}

@inproceedings{chen_LLaGALargeLanguage_2024,
  title = {{{LLaGA}}: {{Large Language}} and {{Graph Assistant}}},
  shorttitle = {{{LLaGA}}},
  booktitle = {International {{Conference}} on {{Machine Learning}}},
  author = {Chen, Runjin and Zhao, Tong and Jaiswal, Ajay Kumar and Shah, Neil and Wang, Zhangyang},
  year = {2024},
  pages = {7809--7823},
  publisher = {PMLR},
}

@inproceedings{zhang_GraphTranslatorAligningGraph_2024a,
  title = {{{GraphTranslator}}: {{Aligning Graph Model}} to {{Large Language Model}} for {{Open-ended Tasks}}},
  shorttitle = {{{GraphTranslator}}},
  booktitle = {The {{Web Conference}} 2024},
  author = {Zhang, Mengmei and Sun, Mingwei and Wang, Peng and Fan, Shen and Mo, Yanhu and Xu, Xiaoxiao and Liu, Hong and Yang, Cheng and Shi, Chuan},
  year = {2024},
  month = may,
  urldate = {2025-11-08},
  langid = {english}
}

@inproceedings{sun_DynamicMessagePassing_2024,
  title = {Towards {{Dynamic Message Passing}} on {{Graphs}}},
  booktitle = {Conference on {{Neural Information Processing Systems}}},
  author = {Sun, Junshu and Yang, Chenxue and Ji, Xiangyang and Huang, Qingming and Wang, Shuhui},
  year = {2024},
  month = dec,
  eprint = {2410.23686},
  primaryclass = {cs},
  urldate = {2025-01-27},
  archiveprefix = {arXiv}
}

@inproceedings{devlin_BERTPretrainingDeep_2019,
  title = {{{BERT}}: {{Pre-training}} of {{Deep Bidirectional Transformers}} for {{Language Understanding}}},
  shorttitle = {{{BERT}}},
  booktitle = {Conference of the {{North American Chapter}} of the {{Association}} for {{Computational Linguistics}}: {{Human Language Technologies}}},
  author = {Devlin, Jacob and Chang, Ming-Wei and Lee, Kenton and Toutanova, Kristina},
  editor = {Burstein, Jill and Doran, Christy and Solorio, Thamar},
  year = {2019},
  pages = {4171--4186},
  publisher = {Association for Computational Linguistics},
  address = {Minneapolis, Minnesota},
}

@inproceedings{chen_TextspaceGraphFoundation_2024,
  title = {Text-Space {{Graph Foundation Models}}: {{Comprehensive Benchmarks}} and {{New Insights}}},
  shorttitle = {Text-Space {{Graph Foundation Models}}},
  booktitle = {Advances in {{Neural Information Processing Systems}}},
  author = {Chen, Zhikai and Mao, Haitao and Liu, Jingzhe and Song, Yu and Li, Bingheng and Jin, Wei and Fatemi, Bahare and Tsitsulin, Anton and Perozzi, Bryan and Liu, Hui and Tang, Jiliang},
  year = {2024},
  month = dec,
  volume = {37},
  pages = {7464--7492},
  urldate = {2025-09-21},
  langid = {english}
}

@inproceedings{kong_GOFAGenerativeOneAll_2024,
  title = {{{GOFA}}: {{A Generative One-For-All Model}} for {{Joint Graph Language Modeling}}},
  shorttitle = {{{GOFA}}},
  booktitle = {The {{Thirteenth International Conference}} on {{Learning Representations}}},
  author = {Kong, Lecheng and Feng, Jiarui and Liu, Hao and Huang, Chengsong and Huang, Jiaxin and Chen, Yixin and Zhang, Muhan},
  year = {2024},
  month = oct,
  urldate = {2025-09-21},
  langid = {english}
}

@inproceedings{tang_GraphGPTGraphInstruction_2024a,
  title = {{{GraphGPT}}: {{Graph Instruction Tuning}} for {{Large Language Models}}},
  shorttitle = {{{GraphGPT}}},
  booktitle = {Proceedings of the 47th {{International ACM SIGIR Conference}} on {{Research}} and {{Development}} in {{Information Retrieval}}},
  author = {Tang, Jiabin and Yang, Yuhao and Wei, Wei and Shi, Lei and Su, Lixin and Cheng, Suqi and Yin, Dawei and Huang, Chao},
  year = {2024},
  pages = {491--500},
  publisher = {Association for Computing Machinery},
  address = {New York, NY, USA},
  urldate = {2025-09-21}
}

@inproceedings{mernyei_WikiCSWikipediaBasedBenchmark_2022,
  title = {Wiki-{{CS}}: {{A Wikipedia-Based Benchmark}} for {{Graph Neural Networks}}},
  shorttitle = {Wiki-{{CS}}},
  booktitle = {International {{Conference}} on {{Machine Learning Workshop}} on {{Graph Representation Learning}} and {{Beyond}}},
  author = {Mernyei, P{\'e}ter and Cangea, C{\u a}t{\u a}lina},
  year = {2022},
  month = jan,
  eprint = {2007.02901},
  primaryclass = {cs},
  publisher = {arXiv},
  archiveprefix = {arXiv}
}

@inproceedings{wang_LLMsZeroshotGraph_2024a,
  title = {{{LLMs}} as Zero-Shot Graph Learners: Alignment of {{GNN}} Representations with {{LLM}} Token Embeddings},
  shorttitle = {{{LLMs}} as Zero-Shot Graph Learners},
  booktitle = {Proceedings of the 38th {{International Conference}} on {{Neural Information Processing Systems}}},
  author = {Wang, Duo and Zuo, Yuan and Li, Fengzhi and Wu, Junjie},
  year = {2024},
  volume = {37},
  pages = {5950--5973},
  publisher = {Curran Associates Inc.},
  address = {Red Hook, NY, USA},
  urldate = {2026-01-22}
}

@inproceedings{wang_UniGTEUnifiedGraph_2025,
  title = {{{UniGTE}}: {{Unified Graph}}--{{Text Encoding}} for {{Zero-Shot Generalization}} across {{Graph Tasks}} and {{Domains}}},
  shorttitle = {{{UniGTE}}},
  booktitle = {The {{Thirty-ninth Annual Conference}} on {{Neural Information Processing Systems}}},
  author = {Wang, Duo and Zuo, Yuan and Lu, Guangyue and Wu, Junjie},
  year = {2025},
  urldate = {2026-01-22},
  langid = {english}
}

@article{wu_MoleculeNetBenchmarkMolecular_2018,
  title = {{{MoleculeNet}}: A Benchmark for Molecular Machine Learning},
  shorttitle = {{{MoleculeNet}}},
  year = {2018},
  journal = {Chemical Science},
  author = {Wu, Zhenqin and Ramsundar, Bharath and Feinberg, Evan N. and Gomes, Joseph and Geniesse, Caleb and Pappu, Aneesh S. and Leswing, Karl and Pande, Vijay},
  volume = {9},
  number = {2},
  pages = {513--530},
  issn = {2041-6520},
  doi = {10.1039/c7sc02664a},
  urldate = {2026-01-28},
  langid = {american}
}

@article{li_GLBenchComprehensiveBenchmark_2024,
  title = {{{GLBench}}: {{A Comprehensive Benchmark}} for {{Graph}} with {{Large Language Models}}},
  shorttitle = {{{GLBench}}},
  author = {Li, Yuhan and Wang, Peisong and Zhu, Xiao and Chen, Aochuan and Jiang, Haiyun and Cai, Deng and Chan, Victor W. and Li, Jia},
  year = {2024},
  journal = {Advances in Neural Information Processing Systems},
  volume = {37},
  pages = {42349--42368},
  urldate = {2026-01-28},
  langid = {english}
}

@inproceedings{wen_AugmentingLowResourceText_2023,
  title = {Augmenting {{Low-Resource Text Classification}} with {{Graph-Grounded Pre-training}} and {{Prompting}}},
  booktitle = {Proceedings of the 46th {{International ACM SIGIR Conference}} on {{Research}} and {{Development}} in {{Information Retrieval}}},
  author = {Wen, Zhihao and Fang, Yuan},
  year = {2023},
  pages = {506--516},
  publisher = {Association for Computing Machinery},
  address = {New York, NY, USA},
  doi = {10.1145/3539618.3591641},
  urldate = {2026-01-28}
}

@inproceedings{yan_ComprehensiveStudyTextattributed_2023,
  title = {A Comprehensive Study on Text-Attributed Graphs: Benchmarking and Rethinking},
  shorttitle = {A Comprehensive Study on Text-Attributed Graphs},
  booktitle = {Proceedings of the 37th {{International Conference}} on {{Neural Information Processing Systems}}},
  author = {Yan, Hao and Li, Chaozhuo and Long, Ruosong and Yan, Chao and Zhao, Jianan and Zhuang, Wenwen and Yin, Jun and Zhang, Peiyan and Han, Weihao and Sun, Hao and Deng, Weiwei and Zhang, Qi and Sun, Lichao and Xie, Xing and Wang, Senzhang},
  year = {2023},
  pages = {17238--17264},
  publisher = {Curran Associates Inc.},
  address = {Red Hook, NY, USA},
  urldate = {2026-01-28}
}

@inproceedings{he_UniGraphLearningUnified_2025,
  title = {{{UniGraph}}: {{Learning}} a {{Unified Cross-Domain Foundation Model}} for {{Text-Attributed Graphs}}},
  shorttitle = {{{UniGraph}}},
  booktitle = {Proceedings of the 31st {{ACM SIGKDD Conference}} on {{Knowledge Discovery}} and {{Data Mining V}}.1},
  author = {He, Yufei and Sui, Yuan and He, Xiaoxin and Hooi, Bryan},
  year = {2025},
  series = {{{KDD}} '25},
  pages = {448--459},
  publisher = {Association for Computing Machinery},
  address = {New York, NY, USA},
  doi = {10.1145/3690624.3709277},
  urldate = {2026-01-28}
}

@inproceedings{lv_GraphPrompterMultiStageAdaptive_2025,
  title = {{{GraphPrompter}}: {{Multi-Stage Adaptive Prompt Optimization}} for {{Graph In-Context Learning}}},
  shorttitle = {{{GraphPrompter}}},
  booktitle = {2025 {{IEEE}} 41st {{International Conference}} on {{Data Engineering}} ({{ICDE}})},
  author = {Lv, Rui and Zhang, Zaixi and Zhang, Kai and Liu, Qi and Gao, Weibo and Liu, Jiawei and Yan, Jiaxia and Yue, Linan and Yao, Fangzhou},
  year = {2025},
  pages = {3917--3930},
  publisher = {IEEE Computer Society},
  doi = {10.1109/ICDE65448.2025.00292},
  urldate = {2026-01-28},
  langid = {english}
}

@article{zheng_JudgingLLMJudgeMTBench_2023,
  title = {Judging {{LLM-as-a-Judge}} with {{MT-Bench}} and {{Chatbot Arena}}},
  author = {Zheng, Lianmin and Chiang, Wei-Lin and Sheng, Ying and Zhuang, Siyuan and Wu, Zhanghao and Zhuang, Yonghao and Lin, Zi and Li, Zhuohan and Li, Dacheng and Xing, Eric and Zhang, Hao and Gonzalez, Joseph E. and Stoica, Ion},
  year = {2023},
  journal = {Advances in Neural Information Processing Systems},
  volume = {36},
  pages = {46595--46623},
  urldate = {2026-01-28},
  langid = {english}
}

@inproceedings{hu_LoRALowRankAdaptation_2021,
  title = {{{LoRA}}: {{Low-Rank Adaptation}} of {{Large Language Models}}},
  shorttitle = {{{LoRA}}},
  booktitle = {International {{Conference}} on {{Learning Representations}}},
  author = {Hu, Edward J. and Shen, Yelong and Wallis, Phillip and {Allen-Zhu}, Zeyuan and Li, Yuanzhi and Wang, Shean and Wang, Lu and Chen, Weizhu},
  year = {2021},
  urldate = {2026-01-28},
  langid = {english}
}
\bibliographystyle{icml2026}

\newpage
\appendix
\onecolumn
\renewcommand{\thefigure}{A\arabic{figure}}
\renewcommand{\thetable}{A\arabic{table}}
\renewcommand{\theequation}{A\arabic{equation}}
\section{Proof}\label{sec:app-proof}
\textbf{Theorem}~\ref{thrm:exist}~(Direction-preserving Update).
\textit{Let $\Delta\mathbf{W}=\texttt{f}(\mathbf{h}; \mathbf{\Theta})$. For a fixed input $\mathbf{h}$ and the Jacobian $J_\texttt{f}(\mathbf{\Theta})=\partial\texttt{f}(\mathbf{h}; \mathbf{\Theta})/\partial\mathbf{\Theta}$, there exists a nonzero parameter direction $\delta\mathbf{\Theta}$ such that under the update $\mathbf{\Theta}\leftarrow\mathbf{\Theta}+\epsilon\delta\mathbf{\Theta}$, the induced weight update satisfies $\Delta\mathbf{W}\leftarrow(1+\epsilon)\Delta\mathbf{W}+O(\epsilon^2)$.}

\begin{proof}
Based on the first-order Taylor expansion
\begin{equation}\label{eq:f-expansion}
\texttt{f}(\mathbf{h};\mathbf{\Theta}+\epsilon\delta\Theta)=\texttt{f}(\mathbf{h};\mathbf{\Theta})+\epsilon J_\texttt{f}(\mathbf{\Theta})\delta\Theta+O(\epsilon^2).
\end{equation}
Our target is to achieve $\texttt{f}(\mathbf{h};\mathbf{\Theta}+\epsilon\delta\Theta)=(1+\epsilon)\texttt{f}(\mathbf{h};\mathbf{\Theta})$. Substituting Eq.~\ref{eq:f-expansion} into the above equation yields
\begin{equation}
\texttt{f}(\mathbf{h};\mathbf{\Theta})+\epsilon J_\texttt{f}(\mathbf{\Theta})\delta\Theta+O(\epsilon^2)=(1+\epsilon)\texttt{f}(\mathbf{h};\mathbf{\Theta}).
\end{equation}
Ignoring higher-order terms, this equation is reduced to a linear function
\begin{equation}\label{eq:linear}
J_\texttt{f}(\mathbf{\Theta})\delta\Theta=\texttt{f}(\mathbf{h};\mathbf{\Theta}).
\end{equation}
In practice, $\texttt{f}:\mathbb{R}^{d}\rightarrow\mathbb{R}^{p\times q}$ is typically implemented as a single-layer perceptron. Therefore, the number of parameters in $\texttt{f}$ is no smaller than the output dimension $p\times q$. Under this implementation, the Jacobian $J_\texttt{f}(\mathbf{\Theta})$ is full row-rank, there exists a solution $\delta\Theta$ for Eq.~\ref{eq:linear} given any $\texttt{f}(\mathbf{h};\mathbf{\Theta})$. The results also holds when $\texttt{f}(\mathbf{h};\mathbf{\Theta})$ is employed to generate the low-rank matrices for weight updates $\Delta\mathbf{W}=\texttt{f}(\mathbf{h}; \mathbf{\Theta})\mathbf{R}^\top$ or $\Delta\mathbf{W}=\mathbf{L}\texttt{f}(\mathbf{h}; \mathbf{\Theta})$, since left or right multiplication by low-rank matrices does not change the above derivation. This completes the proof.
\end{proof}

\textbf{Corollary}~\ref{thrm:exist-lora}~(Direction-preserving Update on LoRA-style Generation).
\textit{Let $\Delta\mathbf{W}=\mathbf{LR}^\top+\texttt{g}(\mathbf{h}; \mathbf{\Theta})\mathbf{R}^\top$. For a fixed input $\mathbf{h}$ and the Jacobian $J_\texttt{g}(\mathbf{\Theta})=\partial\texttt{g}(\mathbf{h}; \mathbf{\Theta})/\partial\mathbf{\Theta}$, there exists a nonzero parameter direction $\delta\mathbf{\Theta}$ such that under the update $\mathbf{\Theta}\leftarrow\mathbf{\Theta}+\epsilon\delta\mathbf{\Theta}$, the induced weight update satisfies $\Delta\mathbf{W}\leftarrow(1+\epsilon)\Delta\mathbf{W}+O(\epsilon)$.}

\begin{proof}
The result in Theorem~\ref{thrm:exist} directly extends to the LoRA setting. Specifically, Eq.~\ref{eq:linear} can be written as
\begin{equation}\label{eq:gara-linear}
J_\texttt{g}(\mathbf{\Theta})\delta\Theta=\texttt{g}(\mathbf{h};\mathbf{\Theta}).
\end{equation}
When $\texttt{f}(\mathbf{h};\mathbf{\Theta})$ is employed to generate the weight update via $\Delta\mathbf{W}=\mathbf{LR}^\top+\texttt{g}(\mathbf{h}; \mathbf{\Theta})\mathbf{R}^\top=\texttt{f}(\mathbf{h}; \mathbf{\Theta})\mathbf{R}^\top$ or $\Delta\mathbf{W}=\mathbf{LR}^\top+\mathbf{L}\texttt{g}(\mathbf{h}; \mathbf{\Theta})=\mathbf{L}\texttt{f}(\mathbf{h}; \mathbf{\Theta})$, we have $\texttt{f}(\mathbf{h}; \mathbf{\Theta})=\texttt{g}(\mathbf{h}; \mathbf{\Theta})+\mathbf{L}$ or $\texttt{g}(\mathbf{h}; \mathbf{\Theta})+\mathbf{R}^\top$. Since $\mathbf{L}$ and $\mathbf{R}$ is independent of $\Theta$, the Jacobian $J_\texttt{g}(\Theta)=J_\texttt{f}(\Theta)$, and thus remains full row rank. Therefore, Eq.~\ref{eq:gara-linear} also admits a solution for any $\texttt{g}(\mathbf{h}; \mathbf{\Theta})$, empowering the direction-preserving update property to hold in the low-rank scenario.
\end{proof}

\begin{table*}
\caption{\textbf{Dataset Statistics.} ``NC", ``LP", ``GC" denote node classification, link prediction, and graph classification, respectively.}
\label{tab:statistics}
\centering
\begin{sc}
\begin{small}
\begin{tabular}{llrrrrr} 
\hline
Type  & Dataset   & Task   & \#Nodes & \#Edges~  & \#Classes/Taks & \#Graphs  \\ 
\hline
Train & arXiv     & NC, LP & 169,343 & 1,166,243 & 40             & 1         \\
Train & AmazonComputers  & NC, LP & 87,229  & 721,081   & 10             & 1         \\
Train & BookChildren  & NC, LP & 76,875  & 1,554,578 & 24             & 1         \\
Train & FB15K237  & LP     & 14,541  & 310,116   & 237            & 1         \\
Train & ChEMBL    & GC     & 26      & 56        & 1,048          & 365,065   \\ 
\hline
Test  & Instagram & NC     & 11,339  & 144,010   & 2              & 1         \\
Test  & Reddit    & NC     & 33,434  & 198,438   & 2              & 1         \\
Test  & WikiCS    & NC     & 11,701  & 216,123   & 10             & 1         \\
Test  & Cora      & NC     & 25,120  & 91,140    & 70             & 1         \\
Test  & BookHistory   & LP     & 41,551  & 358,574   & 12             & 1         \\
Test  & AmazonPhoto     & LP     & 48,362  & 500,928   & 12             & 1         \\
Test  & ChemBACE  & GC     & 34      & 74        & 1              & 1,513     \\
Test  & ChemHIV   & GC     & 25,51   & 54,95     & 1              & 41,127    \\
\hline
\end{tabular}
\end{small}
\end{sc}
\end{table*}
\section{Details on Experimental Settings}\label{sec:app-setup}
\subsection{Datasets}
GaRA and baselines are fine-tuned on five datasets: arXiv~\cite{hu_OpenGraphBenchmark_2020}, BookChildren~\cite{yan_ComprehensiveStudyTextattributed_2023}, AmazonComputer~\cite{shchur_PitfallsGraphNeural_2019}, FB15K237~\cite{liu_OneAllTraining_2023}, and ChEMBL~\cite{chen_TextspaceGraphFoundation_2024}, encompassing node classification, link prediction, and graph classification tasks. On arXiv, BookChildren, and AmazonComputer, both node classification and link prediction tasks are included. After fine-tuning, GaRA and baseline models are evaluated in a zero-shot setting on unseen datasets, including (Cora~\cite{wen_AugmentingLowResourceText_2023}, WikiCS~\cite{mernyei_WikiCSWikipediaBasedBenchmark_2022}, Reddit~\cite{li_GLBenchComprehensiveBenchmark_2024}, and Instagram~\cite{li_GLBenchComprehensiveBenchmark_2024}) for node classification, (AmazonPhoto and BookHistory)~\cite{yan_ComprehensiveStudyTextattributed_2023} for link prediction, and (ChemHIV and ChemBACE)~\cite{wu_MoleculeNetBenchmarkMolecular_2018} for graph classification. The statistics of these datasets are summarized in Tab.~\ref{tab:statistics}. For data splitting, we follow standard splits for node classification and graph classification from their original papers. For link prediction, we randomly split the data into train/valid/test sets with a ratio of 8:1:1. During training, a subset of instances is sampled from each training dataset to stay under the same experimental setup as previous work~\cite{wang_UniGTEUnifiedGraph_2025}. Specifically, 45,470 instances are sampled from arXiv, 21,888 from BookChildren, 31,378 from AmazonComputers for node classification tasks, 10,000 each from (arXiv, BookChildren, and AmazonComputers) and 29,440 from FB15K237 for link prediction tasks, and 74,242 from ChEMBL for graph classification. The task descriptions for each dataset are presented in Tab.~\ref{tab:description}.

\subsection{Experimental Setup}
GaRA follows UniGTE~\cite{wang_UniGTEUnifiedGraph_2025} to construct the input-injection framework and adds a graph-aware weight generation module to incorporate whole-graph information. Specifically, GaRA employs BERT~\cite{devlin_BERTPretrainingDeep_2019} as the pre-trained LM, and Vicuna-7B~\cite{zheng_JudgingLLMJudgeMTBench_2023} as the LLM backbone. During the model tuning process, the learning rate is set as $2e-4$ for LoRA and the graph-aware weight generator, while the rest of the module follows the same setting as in UniGTE. All experiments are conducted on NVIDIA A100.

\begin{figure*}[htb]
\centering
\includegraphics[width=0.6\textwidth]{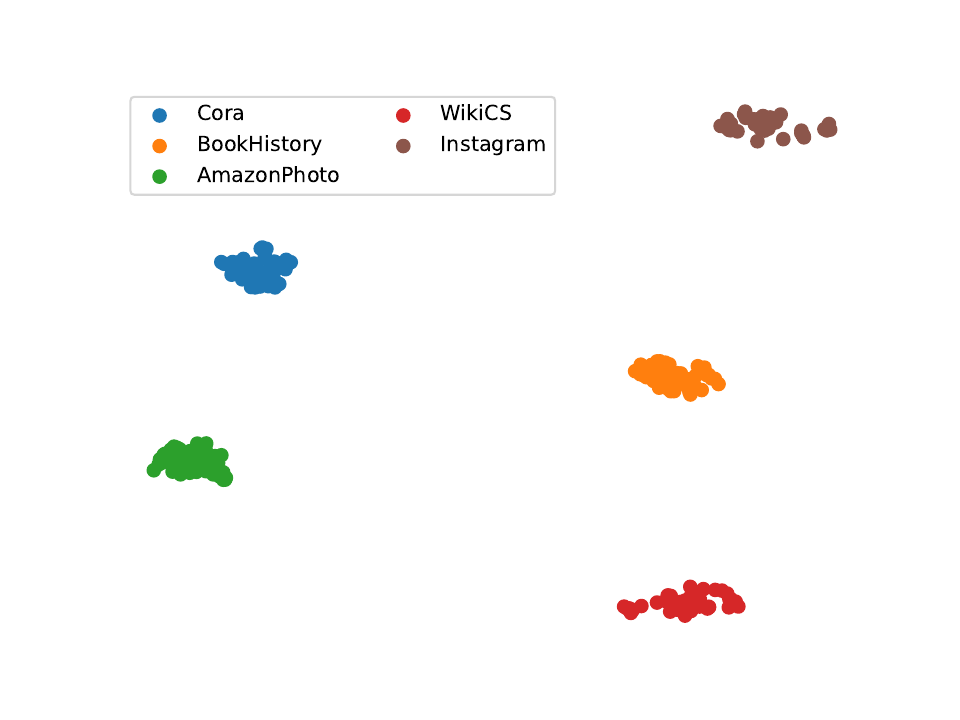}
\caption{\textbf{Distribution of the Generated Weight Updates.} The t-SNE results of the generated weight updates from different datasets are compared. For each dataset, 20\% of the node features are randomly masked to generate 50 perturbed views. Different colors indicate different datasets.}
\label{fig:tsne}
\end{figure*}
\subsection{Graph-awareness of GaRA}
GaRA encodes whole-graph information as weight updates for LLM adaptation. To verify whether the generated updates capture dataset-specific features, we conduct a graph-awareness analysis on the resulting low-rank matrices. Specifically, for each input graph, we randomly mask 20\% of the node features to generate 50 perturbed views of the same graph. The generated low-rank matrices $\mathbf{G}$ from Cora, BookHistory, AmazonPhoto, Instagram, and WikiCS are then visualized using t-SNE~\cite{t-SNE}. As shown in Fig.~\ref{fig:tsne}, different generating results $\mathbf{G}$ corresponding to the same dataset tend to cluster together, while results from different datasets form well-separated clusters. This observation suggests that GaRA encodes dataset-specific graph features, effectively empowering graph-aware adaptation of the LLM across diverse input graphs.

\begin{table*}
\caption{\textbf{Evaluation Results with More Backbones (measured by accuracy for node classification tasks: \%).} ``LP" denotes link prediction.}
\label{tab:abl-backbone}
\centering
\begin{sc}
\begin{small}
\begin{tabular}{lccccccc} 
\hline
                 & \multicolumn{5}{c}{Vicuna-7B}                                                       & \multicolumn{2}{c}{LLaMA3}           \\ 
\cmidrule(lr){2-6}\cmidrule(r){7-8}
                 & \multirow{2}{*}{In} & \multicolumn{4}{c}{GaRA}                                      & \multirow{2}{*}{In} & GaRA           \\
                 &                     & GCN           & GAT           & GIN                    & MLP  &                     & GCN            \\ 
\hline
AmazonPhoto (LP) & 50.3                & \textbf{52.9} & 52.4          & 52.5                   & 51.8 & 51.2                & \textbf{56.6}  \\
BookHistory (LP) & 48.3                & \textbf{59.1} & 58.9          & 58.4                   & 48.6 & 44.1                & \textbf{51.3}  \\
Instagram        & 58.5                & 60.7          & \textbf{64.7} & 63.8                   & 60.0 & 59.1                & \textbf{61.6}  \\
Cora             & 17.2                & 20.8          & 21.1          & \textbf{\textbf{21.8}} & 20.1 & 29.6                & \textbf{31.8}  \\
Reddit           & 50.6                & 51.3          & \textbf{52.8} & 52.3                   & 50.6 & 51.4                & \textbf{51.7}  \\
WikiCS           & 63.2                & 66.4          & \textbf{67.7} & 63.8                   & 63.5 & 65.1                & \textbf{57.1}  \\
\hline
\end{tabular}
\end{small}
\end{sc}
\end{table*}
\subsection{Evaluation with More Backbones}\label{ssec:app-backbone}
To evaluate the compatibility of GaRA with other backbones, GaRA is further implemented with MLP, GAT~\cite{kipf_SemiSupervisedClassificationGraph_2017}, and GIN~\cite{xu*_HowPowerfulAre_2019} for the GNN backbone. LLaMA3.1 8B-Instruct~\cite{llama31_IntroducingLlama31_2024} is adopted as the LLM alternative. Experimental setup follows the ``Partial" setting in Tab.~\ref{tab:abl-norm}. Results in Tab.~\ref{tab:abl-backbone} show that GaRA with different backbones consistently outperforms input-injection baselines, demonstrating its backbone compatibility. Notably, GaRA with different GNNs achieves better performance than the MLP variant. This indicates that the performance gains are from both graph representations and LoRA generation.

\begin{table*}
\caption{\textbf{Comparison with Different Strategies for Constraining Norm Amplification (measured by accuracy for node classification tasks: \%).} ``LP" denotes link prediction. ``No Constraint" denotes GaRA without norm rescaling. }
\label{tab:abl-norm-alter}
\centering
\begin{sc}
\begin{small}
\begin{tabular}{lcccccc} 
\hline
                 & \begin{tabular}[c]{@{}c@{}}No \\Constraint\end{tabular} & \begin{tabular}[c]{@{}c@{}}Norm \\Rescaling \\(GaRA)\end{tabular} & \begin{tabular}[c]{@{}c@{}}Normali\\-zation\end{tabular} & Gating & Penalty & Clipping  \\ 
\hline
AmazonPhoto (LP) & 52.7                                                    & 52.9                                                              & 52.1                                                     & 49.7   & 49.3    & 53.3      \\
BookHistory (LP) & 58.0                                                    & 59.1                                                              & 52.4                                                     & 53.4   & 52.7    & 60.6      \\
Instagram        & 56.9                                                    & 60.7                                                              & 60.1                                                     & 59.2   & 60.4    & 61.3      \\
Cora             & 20.3                                                    & 20.8                                                              & 11.6                                                     & 18.7   & 10.7    & 18.8      \\
Reddit           & 50.2                                                    & 51.3                                                              & 50.6                                                     & 51.0   & 50.7    & 50.2      \\
WikiCS           & 61.1                                                    & 66.4                                                              & 18.7                                                     & 53.7   & 57.5    & 63.3      \\
\hline
\end{tabular}
\end{small}
\end{sc}
\end{table*}
\subsection{Strategies for Constraining Norm Amplification}\label{ssec:app-abl-norm}
Except for norm rescaling, there are other potential strategies for constraining norm amplification on the generated low-rank weight update $\mathbf{G}$. Specifically, we consider loss penalty by adding a loss term (Penalty), normalization by applying Frobenius normalization on $\mathbf{G}$ (Norm), element clipping by clipping elements in $\mathbf{G}$ that are larger than the largest element in the left low-rank matrix $\mathbf{L}$ from LoRA (Clipping), and learnable gating by multiplying a learnable scalar with $\mathbf{G}$ (Gating). Experimental setup follows the "Partial" setting in Tab.~\ref{tab:abl-norm}. Tab.~\ref{tab:abl-norm-alter} shows that GaRA with norm rescaling consistently outperforms Gating, Penalty, and Norm, while being comparable with Clipping, which also constrains the norm with the LoRA component $\mathbf{L}$. Therefore, both norm rescaling and clipping preserve useful norm information while preventing excessive norm amplification. These results suggest that effective constraints control the norm value and preserve meaningful scale information. Our norm rescaling follows these principles, while Clipping serves as a potential alternative. We will explore this direction in future work.

\section{Limitation}\label{sec:app-limitation}
In this paper, we introduce a weight-level injection paradigm that incorporates whole-graph information into LLM adaptation via weight generation. This paradigm is designed as a complementary strategy to input-level injection, which primarily relies on subgraph sequences and discards higher-level information from the original graphs. However, the proposed weight-level paradigm has not yet been explored as a fully independent injection mechanism. In practice, the implemented model, GaRA, is built upon an input-level injection framework and leverages both input- and weight-level information. How to enable LLMs to perform graph tasks using only the weight-level injection paradigm, without relying on input-level prompts, remains an open problem and an interesting direction for future work.

\begin{table*}
\caption{\textbf{Dataset Description.} ``NC", ``LP", ``GC" denotes node classification, link prediction, and graph classification, respectively.}
\label{tab:description}
\centering
\begin{small}
\begin{tabularx}{\textwidth}{lrX} 
\toprule
Datasets                                      & Task Type & Description                                                                                                                                                                                                                                                                                                                                                             \\ 
\hline
\multirow{2}{*}{arXiv, Cora}                  & NC        & Given a representation of a paper with the following information: Title: \{title\}, Abstract: \{abstract\}. Question: Which arXiv CS sub-category does this paper belong to? Please directly give the most likely answer from the following sub-categories: \{candidate\_labels\}.                                                                                \\
                                              & LP        & Given the representation of two papers: Title: First Paper: \{title\}, Second Paper: \{title\}. Question: Do these two papers have citation relationships? Please choose the most likely answer from: "Yes, they have citation relationships" or "No, they do not have citation relationships".                                                                 \\
                                              \hline
\multirow{2}{*}{BookHistory, BookChildren}    & NC        & Given a representation of a book with the following information: Name: \{title\}, Content: \{abstract\}. Question: Which category does this book belong to? Please directly give the most likely answer from the following categories: \{candidate\_labels\}.                                                                                                     \\
                                              & LP        & Given the representation of two books: Title: First Book: \{title\}, Second Book: \{title\}. Question: Do these two books have co-purchased or co-viewed relationships? Choose from: "Yes, they have co-purchased or co-viewed relationships" or "No, they do not have co-purchased or co-viewed relationships".                                                \\
                                              \hline
\multirow{2}{*}{AmazonComputers, AmazonPhoto} & NC        & Given a representation of a book with the following information: Name: \{title\}, Content: \{abstract\}. Question: Which category does this book belong to? Please directly give the most likely answer from the following categories: \{candidate\_labels\}.                                                                                                     \\
                                              & LP        & Given the representation of two electronic products: Title: First Product: \{comment\}, Second Product: \{comment\}. Question: Do these two products have co-purchased or co-viewed relationships? Choose from: "Yes, they have co-purchased or co-viewed relationships" or "No, they do not have co-purchased or co-viewed relationships".                     \\
                                              \hline
WikiCS                                        & NC        & Given a representation of a Wikipedia page with the following information: Name: \{name\}, Content: \{content\}. Question: Which category does this Wikipedia page belong to? Please directly give the most likely answer from the following categories: \{candidate\_labels\}.                                                                                   \\
\hline
Reddit                                        & NC        & Given a representation of a user with the following information: Previous posts: \{posts\}. Question: Which category does this user belong to? Please directly give the most likely answer from the following categories: \{candidate\_labels\}.                                                                                                                  \\
\hline
Instagram                                     & NC        & Given a representation of a user with the following information: Personal introduction: \{introduction\}. Question: Which category does this user belong to? Please directly give the most likely answer from the following categories: \{candidate\_labels\}.                                                                                                    \\
\hline
FB15K237                                      & LP        & Given the representation of two entities: First entity: Name: \{name\}, Description: \{description\}. Second entity: Name: \{name\}, Description: \{description\}. Question: Which category should the relation between these two entities be classified as? Please directly give the most likely answer from the following categories: \{candidate\_labels\}.  \\
\hline
ChEMBL, ChemHIV, ChemBACE                     & GC        & Given a representation of a molecule with the following information: SMILES: \{smiles\}. Question: \{task\} Please answer: "Yes, this molecule is effective to this assay" or "No, this molecule is not effective to this assay".                                                                                                                                 \\
\hline
\multicolumn{2}{r}{NC}                                    & Determine this node’s most likely category within the network’s classification schema                                                                                                                                                                                                                                                                                 \\
\multicolumn{2}{r}{LP}                                    & Determine whether there is a specific relationship between these two nodes.                                                                                                                                                                                                                                                                                           \\
\multicolumn{2}{r}{GC}                                    & Determine whether the molecule possesses specific physicochemical or bioactivity properties.                                                                                                                                                                                                                                                                          \\
\bottomrule
\end{tabularx}
\end{small}
\end{table*}


\end{document}